\documentclass{article}

\usepackage{arxiv}

\usepackage{times,url}
\usepackage{changepage}
\usepackage{algorithmic}
\usepackage{algorithm}
\usepackage{amssymb,amsmath,amsthm}
\usepackage{graphicx}
\usepackage{epsfig}
\usepackage{subfigure}
\usepackage{lipsum}
\usepackage{footnote}
\usepackage{enumitem}
\usepackage{array}

\newtheorem{theorem}{Theorem}
\newtheorem{lemma}[theorem]{Lemma}

\title{Unsupervised Representation Learning with Minimax Distance Measures\thanks{This manuscript is extension of the previous works in \cite{ChehreghaniSDM2016,ChehreghaniAAAI17}.}}
\author{Morteza Haghir Chehreghani  \\
                Department of Computer Science and Engineering \\
                Chalmers University of Technology\\
                SE-412 96 Gothenburg, Sweden\\
                \texttt{morteza.chehreghani@chalmers.se}
}

\begin{document}
\maketitle

\begin{abstract}
We investigate the use of Minimax distances to extract in a nonparametric way the features that capture the unknown underlying patterns and structures in the data. We develop a general-purpose and computationally efficient framework to employ Minimax distances with many machine learning methods that perform on numerical data.
We study both computing the pairwise Minimax distances for all pairs of objects and as well as computing the Minimax distances of all the objects to/from a fixed (test) object.

We first efficiently compute  the pairwise Minimax distances between the objects, using the equivalence of Minimax distances over a graph and over a minimum spanning tree constructed on that. Then, we perform an embedding of the pairwise Minimax distances into a new vector space, such that their squared Euclidean distances in the new space equal to the pairwise Minimax distances in the original space.
We also study the case of having multiple pairwise Minimax matrices, instead of a single one. Thereby, we propose an embedding via first  summing up the centered matrices and then performing an eigenvalue decomposition to obtain the relevant features.

In the following,  we study computing Minimax distances from a fixed (test) object which can be used for instance in $K$-nearest neighbor search. Similar to the case of all-pair pairwise Minimax distances, we develop an efficient and general-purpose algorithm that is applicable with any arbitrary base distance measure. Moreover, we investigate in detail the edges selected by the Minimax distances and thereby explore the ability of Minimax distances in detecting outlier objects.

Finally, for each setting, we perform several experiments to demonstrate the effectiveness of our framework.
\end{abstract}

\section{Introduction}\label{sec:introduction}
Data is usually described by a set of objects and a corresponding representation. The representation can be for example the vectors in a vector space  or the pairwise dissimilarities between the objects.
In real-world applications, the data is often very complicated and a priori unknown. Thus, the basic representation, e.g., \emph{Euclidean} distance, \emph{Mahalanobis} distance, \emph{cosine} similarity and \emph{Pearson correlation}, might fail to correctly capture the underlying patterns or classes. Thereby, the raw data needs to be processed further in order to obtain a more sophisticated representation.
Kernel methods are a common approach to enrich the basic representation of the data and model the underlying patterns \cite{KernelbookShaweTaylor,Hofmann06areview}.
However, the applicability of kernels is confined by several limitations, such as, i) finding the optimal parameter(s) of a kernel function is often very critical and nontrivial~\cite{Nadler07,Luxburg:2007}, and ii) as we will see later, kernels assume a global structure which does not distinguish between the different type of classes in the data.

A category of distance measures, called \emph{link-based} measure~\cite{Fouss:2007,Chebotarev:2011}, takes into account all the \emph{paths} between the objects represented in a graph (where the edge weights indicate the respective pairwise dissimilarities). The \emph{path-specific} distance between nodes $i$ and $j$ is computed by summing the
edge weights on this path~\cite{YenSMS08}. Their link-based distance is then obtained by summing up the \emph{path-specific} measures of all
paths between them. Such a distance measure is known to better capture the arbitrarily shaped patterns compared to the basic representations such as Euclidean or Mahalobis distances.
Link-based measures are often obtained by inverting the Laplacian of the distance matrix, in the context of regularized Laplacian kernel and Markov diffusion kernel~\cite{YenSMS08,Fouss:2012}.
However, computing all-pairs link-based distances requires $\mathcal O(N^3)$ runtime, where $N$ is the number of objects; thus it is not applicable to large-scale datasets.

A more effective distance measure, called \emph{Minimax} measure, selects the minimum largest gap among all possible paths between the objects. This measure, known also as \emph{Path-based} distance measure, has been first investigated on clustering applications \cite{FischerB03,Chehreghani16MLj,PavanP07}. It was also  proposed as an axiom for evaluating clustering methods \cite{ZadehB09}.
A straightforward approach to compute all-pairs Minimax distances is to use an adapted variant of the Floyd-Warshall algorithm. The runtime of this algorithm is $\mathcal O(N^3)$~\cite{Aho:1974,Cormen:2001:IA:580470}.  This distance measure has been also integrated into a variant of $K$-means providing an agglomerative algorithm whose runtime is $\mathcal O(N^2|E|+N^3\log N)$~\cite{FischerB03} ($|E|$ indicates the number of edges in the corresponding graph).

In addition, Minimax distances have been so far applied  to a limited type of classification, i.e. to $K$-nearest neighbor search. The method in~\cite{KimC07icml} presents a message passing algorithm with forward and backward steps, similar to the sum-product algorithm~\cite{Kschischang:2006}. The method takes $\mathcal O(N)$ time, which is in theory equal to the standard $K$-nearest neighbor search, but the algorithm needs several visits of the training dataset. Moreover, this method requires computing a minimum spanning tree (MST) in advance which might require $\mathcal O(N^2)$ runtime. Later on, a greedy algorithm~\cite{KimC13AAAI}, proposes to compute the Minimax $K$ nearest neighbors by space partitioning and using Fibonacci heaps whose runtime is $\mathcal O(\log N + K\log K)$. However, this method is applicable only to Euclidean spaces and assumes the graph is sparse.

\textbf{Motivation.}
Minimax distances enable to cope with arbitrarily shaped  classes and structures in the data. For example, it has been shown that  Minimax $K$-nearest neighbor classification is effective on non-spherical data, whereas the standard variant, the metric learning approach~\cite{Weinberger:2009}, or the shortest path distance~\cite{tenenbaum_global_2000} might give poor results, since they ignore the underlying geometry. See for example Figure 1 in~\cite{KimC13AAAI}.
In particular, four properties of Minimax distances are attractive for us:

\begin{itemize}[leftmargin=*]
\item They enable to compute the patterns and structures in a non-parametric way, i.e. unlike many kernel methods, they do not require fixing any critical parameter in advance.

\item They extract the structures adaptively, i.e., they adapt appropriately whenever the classes differ in shape or type.

\item They take into account the transitive relations: if object $a$ is similar to $b$, $b$ is similar to $c$, ..., to $z$, then the Minimax distance between $a$ and $z$ will be small, although their direct distance might be large. This property is particularly useful when dealing with elongated or arbitrarily shaped patterns.
Moreover, if a basic pairwise similarity is broken due to noise, then their Mimimax distance is able to correct it via taking into account the other paths and relations. For example if the similarity of $i$ and $j$ is broken, then according to Minimax distances, they might still be similar via $i \rightarrow k \rightarrow l \rightarrow … \rightarrow j$.

\item Many learning methods perform on a vector representation of the objects. However, such a representation might not be available. We might be given only the pairwise distances which do not necessarily induce a \emph{metric}. Minimax distances satisfy the \emph{metric} conditions and enable to compute an embedding, as we will study in this paper.
\end{itemize}

\textbf{Contributions.}
Our goal is to develop a generic and computationally efficient framework wherein many different machine learning algorithms can be applied to Minimax distances, beyond e.g., $K$-nearest neighbor classification or the few clustering methods mentioned before.
 Within this unified framework, we consider computing both the all-pair pairwise Miniamx distances and the Minimax distances of all the objects to/from a fixed (test) object.

\begin{enumerate}[leftmargin=*]
\item We first efficiently compute  the pairwise Minimax distances between the objects, using the equivalence of Minimax distances over a graph and over a minimum spanning tree constructed on that.  This approach reduces the runtime of computing the pairwise Minimax distances to $\mathcal O(N^2)$ from $\mathcal O(N^3)$.

\item Then, we investigate the possibility of embedding the pairwise Minimax distances into a vector space.  This feasibility, allows us to  perform an Euclidean  embedding of the pairwise Minimax distances, such that the pairwise squared Euclidean distances in the new space equal to the pairwise Minimax distances in the original space. Such an embedding enables us to apply any numerical learning algorithm on the resultant Minimax vectors.

\item We also consider the cases where there are multiple pairwise Minimax matrices instead of a single matrix, which might happen when dealing with multiple pairwise relations, robustness or high-dimensional data.
      Hence,  to obtain a collective embedding, we propose to first center the individual Minimax matrices and then sum them up. This makes an embedding feasible, because the resultant matrix is positive semidefinite.

\item In the following, we study computing Minimax distances from a fixed (test) object which can be used for instance in $K$-nearest neighbor search. For this case, we propose an efficient and computational optimal Minimax $K$-nearest neighbor algorithm whose runtime is $\mathcal O(N)$, similar to the standard $K$-nearest neighbor search, and it can be applied with any arbitrary base dissimilarity measure.

\item Moreover, we investigate in detail the edges selected by the Minimax distances and thereby explore the ability of Minimax distances in detecting outlier objects.

\item Finally, we experimentally study our framework in different machine learning problems (classification, clustering and $K$-nearest neighbor search) on several  synthetic and real-world datasets and illustrate its effectiveness and superior performance in different settings.
\end{enumerate}

The rest of the paper is organized as following. In Section \ref{sec:notations}, we introduce the notations and definitions. In Section \ref{sec:minimax-general}, we develop our framework for computing pairwise Minimax distances and extracting the relevant features applicable to general machine learning methods. Here, we also extend our approach for computing and embedding Minimax vectors for multiple pairwise data relations, i.e., to collective Minimax distances. In Section \ref{sec:knn-minimax}, we extend further our framework for computing the Minimax $K$-nearest neighbor search and outlier detection. In Section \ref{sec:experiments}, we describe the experimental studies, and finally, we conclude the paper in Section \ref{sec:conclusion}.

\section{Notations and Definitions} \label{sec:notations}
A dataset can be modeled by a graph $\mathcal G(\mathbf O,\mathbf D)$, where $\mathbf O$ and $\mathbf D$ respectively indicate the set of $N$ objects (nodes) and the corresponding edge weights such that $\mathbf D_{ij}$ shows the pairwise dissimilarity between objects $i$ and $j$. The base pairwise dissimilarities $\mathbf D_{ij}$ can be computed for example according to squared Euclidean distances or cosine (dis)similarities between the vectors that represent $i$ and $j$. In several applications, $\mathbf D_{ij}$ might be given directly.\footnote{For simplicity of explanation, we assume that the graph is full, i.e. the missing edges are filled by a large value. However, our analysis can be easily extended to arbitrary graphs \cite{Chehreghani17ICDM}.}
In general, $\mathbf D$ might not yield a \emph{metric}, i.e. the triangle inequality may not hold.
We have recently studied the application of Minimax distances to correlation clustering, where the triangle inequality does not necessarily hold for the base pairwise dissimilarities \cite{abs-1904-13223}.
In our study, $\mathbf D$ needs to satisfy three basic conditions:

\begin{enumerate}[leftmargin=*]
\item zero self distances, i.e. $\forall i,\mathbf D_{ii}=0$,
\item non-negativity, i.e. $\forall i,j, \mathbf D_{ij}\ge 0$, and
\item symmetry, i.e. $\forall i,j, \mathbf D_{ij}=\mathbf D_{ji}$.
\end{enumerate}

We  assume there are no duplicates, i.e. the pairwise dissimilarity between every two distinct objects is positive. For this purpose, we may either remove the duplicate objects or perturb them slightly to make the zero non-diagonal elements of $\mathbf D$ positive. The goal is then to use the Minimax distances and the respective features to perform the machine learning task of interest.

The Minimax (MM) distance between objects $i$ and $j$ is defined as

\begin{eqnarray}
	\mathbf D_{i,j}^{MM} &=& \min_{r\in \mathcal R_{ij}(\mathbf O)}\{ \max_{1\le l \le |r|-1}\mathbf D_{r(l)r(l+1)}\},
	\label{eq:PathStandard}
\end{eqnarray}
where $\mathcal R_{ij}(\mathbf O)$ is the set of all paths between $i$ and $j$. Each path $r$ is identified by a sequence of object indices, i.e. $r(l)$
shows the $l^{th}$ object on the path.

\section{Features Extraction from Pairwise Minimax Distances} \label{sec:minimax-general}

We aim at developing a  unified framework for performing arbitrary numerical learning methods with Minimax distances.  To cover different algorithms, we pursue the following strategy:
\begin{enumerate}[leftmargin=*]
\item We compute the pairwise Minimax distances for all pairs of objects $i$ and $j$ in the dataset.
\item We, then, compute an embedding of the objects into a new vector space such that their pairwise (squared Euclidean) distances in this space equal their Minimax distances in the original space.
\end{enumerate}

Notice that vectors are the most basic way for data representation, since they render a bijective mapping between the objects and the measurements. Hence, any machine leaning method which performs on numerical data can benefit from our approach. Such methods perform either on the objects in a vector space, e.g. logistic regression, or on a kernel matrix computed from the pairwise relations between the objects. In the later case, the pairwise Minimax distances can be used for this purpose, for example through an exponential transformation, or the kernel can be computed from the final Minimax vectors.

\subsection{Computing pairwise Minimax distances}
Previous works for computing Minimax distances (e.g. applied to clustering) use a variant of Floyd-Warshall algorithm whose runtime is $\mathcal O(N^3)$~\cite{Aho:1974,Cormen:2001:IA:580470}, or combine it with $K$-means in an agglomerative algorithm whose runtime is $\mathcal O(N^2|E|+N^3\log N)$~\cite{FischerB03}. To reduce such a computational demand, we follow a more efficient procedure established in two steps: i) build a minimum spanning tree (MST) over the graph, and then, ii) compute the Minimax distances over the MST.

\noindent\textbf{I. Equivalence of Minimax distances over a graph and over a minimum spanning tree on the graph}

We exploit the equivalence of Minimax distances over an arbitrary graph and those obtained from a minimum spanning tree on the graph, as expressed in Theorem \ref{theorem:PathGraphMST}.\footnote{This result may lead to several simplifications for computing the base pairwise distance measure $\mathbf D$ and the respective graph $\mathcal G(\mathbf O, \mathbf D)$, where the graph does not need to be necessarily full.  For example, instead of a full graph, we can compute a minimal (connected) graph that  sufficiently includes a minimum spanning tree. For instance, one might use the $K$ nearest neighbor graph, which is not necessarily full.
On the other hand, we might have no idea about the pairwise dissimilarities. Then we may query them from a user and infer the rest by Minimax distances. In fact, this result proposes an efficient method for this purpose: due to equivalence of Minimax distances on a graph and on any minimum spanning tree on that, we may be able to query significantly smaller number of pairwise dissimilarities that sufficiently represent a MST. This can be significantly faster than querying or inferring all the pairwise distances to obtain $\mathbf D$.}

\begin{theorem}
Given graph $\mathcal G(\mathbf O,\mathbf D)$, a minimum spanning tree constructed on that provides the sufficient and necessary edges to obtain the pairwise Minimax distances.
\label{theorem:PathGraphMST}
\end{theorem}
\begin{proof}

\begin{figure}[t!]
    \centering
    \subfigure[$e^{\texttt{MM}}_{i,j}$ is the largest edge between $i$ and $j$.]
    {
        \includegraphics[width=0.34\textwidth]{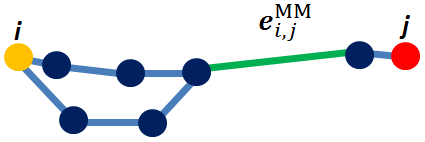}
        \label{fig:MSTPath1route}
    }
    \hspace{7mm}
    \subfigure[$e^{\texttt{MM}}_{i,j}$ is not the largest edge between $i$ and $j$.]
    {
        \includegraphics[width=0.34\textwidth]{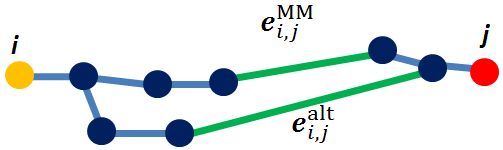}
        \label{fig:MSTPath2route}
    }
    \caption{Equivalence of Minimax distances over a graph and over a MST constructed on the graph.}
    \label{fig:MSTPath}
\end{figure}

The equivalence is concluded in a similar way to the \emph{maximum capacity} problem~\cite{Hu61}, where one can show that picking an edge which does not belong to a minimum spanning tree, yields a larger Minimax distance (i.e., a contradiction occurs).

We here describe a high level proof. Let $e^{\texttt{MM}}_{ij}$ denote the edge representing the Minimax distance between $i$ and $j$. We investigate two cases:

\begin{enumerate} [leftmargin = *]
\item There is only one path between $i$ and $j$. Then, this path will be part of any minimum spanning tree constructed over the graph. Since, otherwise, if some edge for example $e^{\texttt{MM}}_{ij}$ is not selected, then the tree will lose the connectivity, i.e. there will be less edges in the tree than $N-1$. The same argument holds when there are several paths between $i$ and $j$, but all share  $e^{\texttt{MM}}_{ij}$ such that $e^{\texttt{MM}}_{i,j}$ is the largest edge on all of them (Figure~\ref{fig:MSTPath1route}).  Then $e^{\texttt{MM}}_{ij}$ will be selected by any minimum spanning tree, as otherwise the tree will be disconnected.

\item There are several paths between $i$ and $j$, whose largest edges are different.  We need to show that only the path including $e^{\texttt{MM}}_{i,j}$ is selected by a minimum spanning tree. To prove, consider two paths, one including $e^{\texttt{MM}}_{i,j}$ and the other containing $e^{\texttt{alt}}_{i,j}$ which is the largest edge on the alternative path (Figure~\ref{fig:MSTPath2route}). It is easy to see that the minimum spanning tree would choose $e^{\texttt{MM}}_{i,j}$ instead of $e^{\texttt{alt}}_{i,j}$: existence of two different paths indicates presence of at least one cycle. According to the definition of a tree, one edge  of a cycle should be eliminated to yield a tree.  In the cycle including $e^{\texttt{MM}}_{i,j}$, according to the definition of Minimax distances, there is at least one edge not smaller than $e^{\texttt{MM}}_{i,j}$ (which is $e^{\texttt{alt}}_{i,j}$), as otherwise $e^{\texttt{MM}}_{i,j}$ would not represent the Minimax distance (the minimum of largest gaps) between $i$ and $j$. Thereby the minimum spanning tree algorithm keeps  $e^{\texttt{MM}}_{i,j}$ instead of $e^{\texttt{alt}}_{i,j}$, since this choice leads to a tree with a smaller total weight. Notice that the algorithm cannot discard both $e^{\texttt{MM}}_{i,j}$ and $e^{\texttt{alt}}_{i,j}$,  because then the connectivity will be lost.
\end{enumerate}

On the other hand, it is \emph{necessary} to build a minimum spanning tree, e.g. to have at least $N-1$ edges selected given that the graph $\mathcal G$ is connected. Otherwise, the computed MST will lose its connectivity and thereby for some pairs of objects there will be no path to be computed.
\end{proof}

Therefore, to compute  the Minimax distances $\mathbf D^{MM}$, we need to care only about the edges which are present in an MST over the graph.  Then, the Minimax distances are written by
\begin{equation}
\mathbf D^{MM}_{ij} = \max_{1\le l \le |r_{ij}|-1}\mathbf D_{r_{ij}(l)r_{ij}(l+1)},\
\end{equation}
where $r_{ij}$ indicates the (only) path between $i$ and $j$. This result does not depend on the particular choice of the algorithm used for constructing the minimum spanning tree. The graph that we assume is a full graph, i.e. we compute the pairwise dissimilarities between all pairs of objects. For full graphs, the straightforward implementation of the Prim's algorithm~\cite{Prim1957} using an auxiliary vector requires $\mathcal O(N^2)$ runtime which is an optimal choice.

\noindent\textbf{II. Pairwise Minimax distances over a MST}

At the next step, after constructing a minimum spanning tree, we compute the pairwise Minimax distances over that.
A naive and straightforward algorithm would perform a Depth First Search (DFS) from each node to compute the Minimax distances by keeping the track of the largest distance between the initial node and each of the traversed nodes. A single run of DFS requires $\mathcal O(N)$ runtime and thus the total time will be $\mathcal O(N^2)$.
However, such a method might lead to visiting  some edges multiple times which renders an unnecessary extra computation. For instance, a maximal edge weight might appear in many DFSs such that it is processed several times. Hence, a more elegant approach would first determine all the objects whose pairwise distances are represented by an edge weight and then assigns the weight as their Minimax distances. According to Theorem~\ref{theorem:PathGraphMST}, every edge in the MST represents some Minimax distances. Thus, we first find the edge(s) which represent only few Minimax distances, namely only one pairwise distance, such that the respective objects can be immediately identified. Lemma~\ref{lemma:PathMSTsmallest} suggest existence of this kind of edges.

\begin{lemma}
In a minimum spanning tree, for the minimal edge weights we have $\mathbf D^{MM}_{ij}=\mathbf D_{ij}$, where $i$ and $j$ are the two objects that occur exactly at the two sides of the edge.
\label{lemma:PathMSTsmallest}
\end{lemma}
\begin{proof}
Without loss of generality, we assume that the edge weights are distinct. Let $e_{min}$ denote the edge  with minimal weight in the MST. We consider the pair of objects $p$ and $q$ such that at least one of them is not directly connected to  $e_{min}$ and show that $e_{min}$ does not represent their Minimax distance.
In a MST, there is exactly one path between each pair of objects. Thus,  on the path between $p$ and $q$ there is at least another edge whose weight is not smaller than the weight of $e_{min}$, which hence represents the Minimax distance between $p$ and $q$, instead of $e_{min}$.
\end{proof}

Lemma~\ref{lemma:PathMSTsmallest} yields a dynamic programming approach to compute the pairwise Minimax distances from a tree. We first  sort the edge weights of the MST via for example merge sort or heap sort~\cite{Cormen:2001:IA:580470} which in the worst case require $\mathcal O(N \log N)$ time.
We then consider each object as a separate component. We process the edge weights one by one from the smallest to the largest, and at each step, we set the Minimax distance of the two respective components  by this weight. Then, we remove the two components and replace them by a new component constructed from the combination (union) of the two. We repeat these two steps for all edges of the minimum spanning tree. A main advantage of this algorithm is that  whenever an edge is processed, all the nodes that it represents their Minimax distances are ready in the components connected to the two sides of the edge.
Algorithm \ref{alg:All_pair_Path-based_Dist} describes the procedure in detail.

\begin{algorithm}[htb!]
\caption{Computation of pairwise Minimax distances on a tree.}
\label{alg:All_pair_Path-based_Dist}
\begin{algorithmic} [1]
\REQUIRE {A tree specified by, i) \emph{indices}: a $|E|\times2$ matrix of (endpoints of) edges, and ii) \emph{weights}: vector of edge weights.}
\ENSURE Matrix of pairwise Minimax distances $\mathbf D^{MM}$.
\vspace{2mm}

\STATE $sorting\_ind  = \arg sort(weights)$
\STATE $weights = weights[sorting\_ind]$
\STATE $indices = indices[sorting\_ind,:]$
\vspace{2mm}

\STATE \emph{component\_list} $= []$
\FOR{$i=0$ \textbf{to} $N-1$}
	\STATE \emph{component\_id}[i] $= i$
	\STATE \emph{component\_list.append}$([i])$
\ENDFOR
\STATE \emph{cur\_id} $= N-1$

\FOR{$i=0$ \textbf{to} $|E|-1$}
\STATE $ind1 = indices[i,0]$
\STATE $ind2 = indices[i,1]$

\IF{\emph{component\_id}$[ind1]$ $\ne$ \emph{component\_id}$[ind2]$}
	\STATE \emph{first\_side = component\_list}$[$\emph{component\_id}$[ind1]]$
	\STATE \emph{scnd\_side = component\_list}$[$\emph{component\_id}$[ind2]]$

	\vspace{2mm}
	\STATE \emph{new\_component= first\_side $\cup$ scnd\_side}
	\STATE \emph{cur\_id = cur\_id +1}

	\vspace{2mm}

	\STATE \emph{component\_id}$[$\emph{new\_component}$]$ \emph{= cur\_id}
	\STATE \emph{component\_list.append}$($\emph{new\_component}$)$
	
	\vspace{2mm}
	
	\STATE $\mathbf D^{MM}[first\_side,scnd\_side] = weights[i]$
	\STATE $\mathbf D^{MM}[scnd\_side,first\_side] = weights[i]$
\ENDIF
\ENDFOR

\RETURN $\mathbf D^{MM}$
\end{algorithmic}
\end{algorithm}

Algorithm \ref{alg:All_pair_Path-based_Dist} uses the following data structures.

\begin{enumerate}[leftmargin =*]
\item $component\_list$: is a list of lists, wherein each list (\emph{component}) contains the set of nodes (objects) that are treated similarly, i.e. they have the same Minimax distance to/from an external node or component.
\item $component\_id$: a $N$-dimensional vector containing the ID of the latest component that each object belongs to.
\end{enumerate}

Each object initially constitutes a separate component.
The algorithm, at each step $i$, pops out the unselected edge $e_i$ which has the smallest weight. For this purpose, we assume that the vector of the edge weights (i.e. $weights$) has been sorted in advance. The nodes associated to the edges are rearranged according to the ordering of $weights$ and are stored in $indices$.
If the edge $e_i$ does not entirely occur inside a single component, then  the nodes reachable from each side of $e_i$ (i.e. from $ind1$ and $ind2$) are selected and stored respectively in $first\_side$ and $scnd\_side$. For this purpose, we use a vector called $component\_id$ which keeps the ID (index) of the latest component that each node belongs to. Therefore, the component  $first\_side$ is obtained by $component\_list[component\_id[ind1]]$ and similarly $scnd\_side$ by \\ $component\_list[component\_id[ind2]]$. Then $\mathbf D^{MM}$ is updated by
\begin{align}
&\mathbf D^{MM}[first\_side,scnd\_side] = weights[i], \nonumber\\
&\mathbf D^{MM}[scnd\_side,first\_side] = weights[i].
\label{eq:updatePathMatrix}
\end{align}

Finally, a new component is constructed and added to $component\_list$ by combining the two base components $first\_side$ and $scnd\_side$. The ID of this new component is used as the ID of its members in $component\_id$.

\subsection{Embedding of pairwise Minimax distances}

In the next step, given the matrix of pairwise Minimax distances $\mathbf D^{MM}$, we obtain an embedding of the objects into a vector space such that their pairwise squared Euclidean distances in this new space are equal to their Minimax distances in the original space. For this purpose, in Theorem~\ref{theorem:MinimaxEmbedding}, we investigate a useful property of Minimax distance measures, called the \emph{ultrametric} property~\cite{Leclerc1981} which can be used to prove the existence of such an embedding.

\begin{theorem}
Given the pairwise distances $\mathbf D$, the matrix of Minimax distances $\mathbf D^{MM}$ induces an $\mathcal{L}_2^2$ embedding, i.e. there exist a new vector space for the set of objects $\mathbf O$ wherein the pairwise squared Euclidean distances are equal to the pairwise Minimax distances in the original space.
\label{theorem:MinimaxEmbedding}
\end{theorem}
\begin{proof}
First, we investigate that the pairwise Minimax distances $\mathbf D^{MM}$ constitute an \emph{ultrametric}. According to~\cite{Leclerc1981}, the conditions to be satisfied are:

\begin{enumerate}[leftmargin=*]
\item $\forall i,j: \mathbf D^{MM}_{ij} = 0$ if and only if $i=j$. We verify each of the conditions separately. i) If $i=j$, then $\mathbf D^{MM}_{ij} = 0$: We have $\mathbf D^{MM}_{ii}=\mathbf D_{ii}=0$ because the smallest maximal gap between every object and itself is zero. ii) If $\mathbf D^{MM}_{ij}=0$, then $\mathbf D_{ij}=0$ and $i=j$, because we have assumed that all the distinct pairwise distances are positive, i.e. zero base or Minimax pairwise distances can occur only if $i=j$.

\item $\forall i,j: \mathbf D^{MM}_{ij} \ge 0$. All the edge weights, i.e. the elements of $\mathbf D$, are non-negative. Thus the minimum of  them, i.e. $\min(\mathbf D)$, is also non-negative. Moreover, by definition we have $\mathbf D^{MM}_{ij} \ge \min(\mathbf D)$. Hence, we conclude $\mathbf D^{MM}_{ij} \ge \min(\mathbf D) \ge 0$.

\item  $\forall i,j: \mathbf D^{MM}_{ij} = \mathbf D^{MM}_{ji}$. By assumption $\mathbf D$ is symmetric, therefore, any path from $i$ to $j$ will also be  a path from $j$ to $i$, and vice versa. Thereby, their maximal weights and the minimum among different paths are identical.

\item $\forall i,j,k: \mathbf D^{MM}_{ij}  \le  \max\{\mathbf D^{MM}_{ik},\mathbf D^{MM}_{kj}\}$. We show that otherwise a contradiction occurs. Suppose there is a triplet $i,j,k$ such that $\mathbf D^{MM}_{ij}  >  \max\{\mathbf D^{MM}_{ik},\mathbf D^{MM}_{kj}\}$. Then, according to the definition of Minimax distance, the path from $i$ to $k$ and then to $j$ must be used for computing the Minimax distance $\mathbf D^{MM}_{ij}$ which leads to $\mathbf D^{MM}_{ij} \le  \max\{\mathbf D^{MM}_{ik},\mathbf D^{MM}_{kj}\}$, i.e. a contradiction occurs.
\end{enumerate}

On the other hand,  \emph{ultrametric} matrices are positive definite \cite{Fiedler1998,Varga1993OnSU} and positive definite matrices induce an Euclidean embedding \cite{Schoenberg}.
\end{proof}

Notice that we do not require $\mathbf D$ to induce a \emph{metric}, i.e. the triangle inequality is not necessarily fulfilled.
After satisfying the feasibility condition, there exist several ways to compute a squared Euclidean embedding. We exploit a method motivated in~\cite{RePEc1938} and further analyzed in~\cite{torgerson1958theory} that is known as \emph{multidimensional scaling} \cite{doi:10.2307/2348634}. This method works based on centering $\mathbf D^{MM}$ to compute a Mercer kernel which is positive semidefinite, and then performing an eigenvalue decomposition, as following:
\footnote{One could instead use more efficient methods such as \cite{GlobersonCPT07,KhoshneshinS10}, or use faster approximate of eigenvectors/eigenvalues \cite{Quarteroni2007}.}
\begin{enumerate}
    \item Center $\mathbf D^{MM}$ by
    \begin{equation}
   	 \mathbf{W}^{MM}\leftarrow -\frac{1}{2}\mathbf{A} \mathbf D^{MM} \mathbf{A}.
	\label{Eq:centering}
    \end{equation}
    $\mathbf{A}$ is defined as $\mathbf{A} = \mathbf{I}_N - \frac{1}{N}\mathbf{e}_N\mathbf{e}_N^{T}$, where $\mathbf{e}_N$ is a vector of length $N$ with $1$'s and $\mathbf{I}_N$ is an identity matrix of size $N$.

    \item Under this transformation, $\mathbf W^{MM}$ is positive semidefinite. Thus, we decompose $\mathbf W^{MM}$ into its eigenbasis, i.e.,
\begin{equation}
 \mathbf W^{MM}=\mathbf{V}\boldsymbol{\Lambda}\mathbf{V}^{T},
\end{equation}
  where $\mathbf{V} = (v_1,...,v_N)$ contains the eigenvectors $v_i$ and $\boldsymbol{\Lambda}=\texttt{diag}(\lambda_1,...,\lambda_N)$ is a diagonal matrix of eigenvalues $\lambda_1\geq...\geq\lambda_d\geq\lambda_{d+1}= 0 = ... = \lambda_N$.  Note that the eigenvalues are nonnegative, since $\mathbf W^{MM}$ is positive semidefinite.

    \item Calculate the $N\times d$ matrix

\begin{equation}
\mathbf{Y}^{MM}_d=\mathbf{V}_d(\boldsymbol{\Lambda}_d)^{1/2},
\end{equation}
where $\mathbf{V}_d=(v_1,...,v_d)$ and   $\boldsymbol{\Lambda}_d=\text{diag}(\lambda_1,...,\lambda_d)$.

\end{enumerate}

Here, $d$ shows the dimensionality of the Minimax vectors, i.e. the number of Minimax features. The Minimax dimensions are ordered according to the respective eigenvalues and thereby we might choose only the first most representative ones, instead of taking them all.

\subsection{Embedding of collective Minimax pairwise matrices}\label{sec:minimax-collective}

We extend our generic framework to the cases where multiple pairwise Minimax matrices are available, instead of a single matrix. Then, the goal would be to find an embedding of the objects into a new space wherein their pairwise squared Euclidean distances are related to the collective Minimax distances over different Minimax matrices. Such a scenario might be interesting in several situations:

\begin{itemize}[leftmargin=*]
\item There might exist different type of relations between the objects, where each renders a separate graph. Then, we compute several pairwise Minimax distances, each for a specific relation.

\item Minimax distances might fail when for example few noise objects connect two compact classes. Then, the inter-class Minimax distances become very small, even if the objects from the two classes are connected via only few outliers. To solve this issue, similar to model averaging, one could use the  higher order, i.e. the second, third, ...  Minimax distances, e.g., the second smallest maximal gap. Then, there will be multiple pairwise Minimax matrices each representing the $k^{th}$ Minimax distance.

\item In many real-world applications, we encounter high-dimensional data, the patterns might be hidden in some unknown subspace instead of the whole space, such that they are disturbed in the high-dimensional space due to the curse of dimensionality. Minimax distances rely on the existence of well-connected paths, whereas such paths might be very sparse or fluctuated in high dimensions.
Thereby, similar to ensemble methods, it is natural to seek for connectivity paths and thus for Minimax distances in some subspaces of the original space.
~\footnote{An alternative approach would be to perform dimension reduction via e.g. Principle Component Analysis (PCA) and so on. Our approach can be easily extended to such cases too.}
However, investigating all the possible subspaces is computationally intractable as the respective cardinality scales exponentially with the number of dimensions. Hence, we
propose an approximate approach based on computing Minimax distances for each dimension, which leads to having multiple Minimax matrices.
\end{itemize}

In all the aforementioned cases, we need to deal with (let say $M$) different matrices of Minimax distances computed for the same set of objects. Then, the next step is to investigate the existence of an embedding that represents the pairwise collective Minimax distances. To proof the existence of such an embedding for $M=1$, we used the \emph{ultrametric} property of the respective Minimac distances and then concluded the positive definiteness. However, as shown in Theorem~\ref{theorem:MinimaxEmbeddingSum},  the sum of $M>1$ Minimax matrices does not necessarily satisfy the \emph{ultrametric} property.

\begin{theorem}
Given $M>1$ pairwise Minimax matrices $\mathbf D^{MM}(m), m\in\{1,..,M\}$ for the same set of objects $\mathbf O$, then, the accumulative Minimax matrix $\mathbf D^{aMM} = \mathbf D^{MM}(1) + ... + \mathbf D^{MM}(m)$ does not necessarily constitute an ultrametric.
\label{theorem:MinimaxEmbeddingSum}
\end{theorem}

\begin{proof}
We investigate the \emph{ultrametric} conditions for  $\mathbf D^{aMM}$ (using the results of the analysis in Theorem 2):
\begin{enumerate}[leftmargin=*]

\item $\forall i,j: \mathbf D^{aMM}_{ij} = 0$, if and only if $i=j$. We investigate each of the conditions.\\ i) If $i=j$, then $\mathbf D^{aMM}_{ij} = 0$: Since we have $\mathbf D^{MM}_{ii}(1)= ..= \mathbf D^{MM}_{ii}(M) = 0$, then $\mathbf D^{aMM}_{ii}= \sum_{m=1}^{M} \mathbf D^{MM}_{ii}(m) = 0$.
\\ ii) If $\mathbf D^{aMM}_{ij}=0$, then $i=j$: According to our assumption, for $i \ne j$ we have  $\mathbf D_{ij}>0$ and thus $\mathbf D^{MM}_{ij}>0$. Therefore, $\mathbf D^{aMM}_{ij}>0$ if $i\ne j$, i.e. $\mathbf D^{aMM}_{ij}=0$ implies that $i=j$.

\item $\forall i,j: \mathbf D^{aMM}_{ij} \ge 0$. We have : $\forall 1\le m\le M, \mathbf D^{MM}_{ij}(m) \ge 0$. Thus, $$\mathbf D^{aMM}_{ij}= \sum_{m=1}^{M} \mathbf D^{MM}_{ij}(m) \ge 0.$$

\item $\forall i,j: \mathbf D^{aMM}_{ij} = \mathbf D^{aMM}_{ji}$. We have: $\forall 1\le m\le M, \mathbf D^{MM}_{ij} (m) = \mathbf D^{MM}_{ji}(m)$. Thus, $\sum_{m=1}^{M}\mathbf D^{MM}_{ij} (m) = \sum_{m=1}^{M}\mathbf D^{MM}_{ji}(m)$, i.e. $\mathbf D^{aMM}_{ij} = \mathbf D^{aMM}_{ji}$.

\item $\forall i,j,k: \mathbf D^{aMM}_{ij}  \le  \max\{\mathbf D^{aMM}_{ik},\mathbf D^{aMM}_{kj}\}$. $\forall 1\le m\le M$, we have $$\mathbf D^{MM}_{ij}(m)  \le  \max\{\mathbf D^{MM}_{ik}(m),\mathbf D^{MM}_{kj}(m)\}.$$ Then,
$$\sum_{m=1}^{M}\mathbf D^{MM}_{ij}(m)  \le \sum_{m=1}^{M}  \max\{\mathbf D^{MM}_{ik}(m),\mathbf D^{MM}_{kj}(m)\}.$$
However, we need to show that
$$\sum_{m=1}^{M}\mathbf D^{MM}_{ij}(m)  \le \max\{\sum_{m=1}^{M} \mathbf D^{MM}_{ik}(m),\sum_{m=1}^{M}\mathbf D^{MM}_{kj}(m)\}.$$
Thus, if we can approximate $\max\{\sum_{m=1}^{M} \mathbf D^{MM}_{ik}(m),\sum_{m=1}^{M} \mathbf D^{MM}_{kj}(m)\}$ by\\ $\sum_{m=1}^{M}  \max\{\mathbf D^{MM}_{ik}(m),\mathbf D^{MM}_{kj}(m)\}$, then, this condition is satisfied too and thereby $\mathbf D^{aMM}$ induces an \emph{ultrametric}.
\\ However, in general this approximation is not valid and one can find cases where $\mathbf D^{aMM}_{ij}  >  \max\{\mathbf D^{aMM}_{ik},\mathbf D^{aMM}_{kj}\}$. Here, we show an example. Let fix $M=2$ and consider the datasets in  Figure~\ref{fig:aMMfail}. The two datasets represent the same set of objects. They differ only in the position of the $k^{th}$ object and the other objects are fixed among the two datasets. In both datasets, the Minimax distance between $i$ and $j$ is equal to $a$, i.e. $\sum_{m=1}^{M}\mathbf D^{MM}_{ij}(m) = 2a$. On the other hand, we have $\sum_{m=1}^{M}\mathbf D^{MM}_{ik}(m)=b+a$,  and similarly $\sum_{m=1}^{M}\mathbf D^{MM}_{kj}(m)= a +b$. Thus, $\sum_{m=1}^{M}\mathbf D^{MM}_{ij}(m) = 2a  > \max\{\sum_{m=1}^{M} \mathbf D^{MM}_{ik}(m),\sum_{m=1}^{M} \mathbf D^{MM}_{kj}(m)\}= a+b$, i.e., $\mathbf D^{aMM}$ does not induce an \emph{ultrametric}.
\end{enumerate}
\end{proof}

\begin{figure}[t!]
    \centering
    \subfigure[First representation.]
    {
        \includegraphics[width=0.26\textwidth]{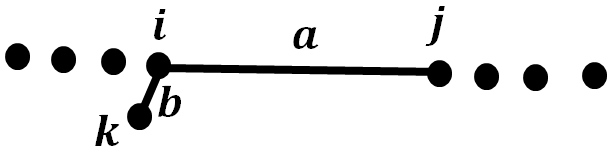}
        \label{fig:aMM1}
    }
    \hspace{3.1mm}
    \subfigure[Second representation.]
    {
        \includegraphics[width=0.26\textwidth]{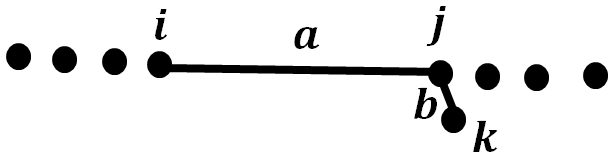}
        \label{fig:aMM2}
    }
    \caption{Two different representation of the same objects. The difference is only in the position of the $k^{th}$ object. Summing up the pairwise Minimax distances of these two datasets does not lead to an \emph{ultrametric}.}
    \label{fig:aMMfail}
\end{figure}

Theorem~\ref{theorem:MinimaxEmbeddingSum} indicates that the accumulative Minimax matrix $\mathbf D^{aMM}$ does not always yield an \emph{ultrametric}.
 However, we do not necessarily need the \emph{ultrametric} property to compute an embedding, rather it is only a sufficient condition.  Theorem~\ref{theorem:MinimaxEmbeddingCollective} suggests that after computing $\mathbf D^{MM}(m), m\in\{1,..,M\}$, first we center each of them via Eq.~\ref{Eq:centering} and then sum them up. The resultant matrix $\mathbf W^{cMM}$ is embeddable as it is positive definite.

\begin{theorem}
For a set of objects $\mathbf O$, we are given $M$ pairwise Minimax matrices  $\mathbf D^{MM}(m), m\in\{1,..,M\}$ and the respective centered matrices $\mathbf W^{MM}(m)$ defined as
\begin{equation}
\mathbf W^{MM}(m) = -\frac{1}{2}\mathbf{A} \mathbf D^{MM}(m) \mathbf{A}.
\end{equation}
Then, the collective matrix $\mathbf W^{cMM} = \sum_{m=1}^{M}\mathbf W^{MM}(m)$ induces an $\mathcal{L}_2^2$ embedding.
\label{theorem:MinimaxEmbeddingCollective}
\end{theorem}

\begin{proof}
According to Theorem \ref{theorem:MinimaxEmbedding}, each $\mathbf D^{MM}(m)$ yields an $\mathcal{L}_2^2$ embedding, which implies that $\mathbf W^{MM}(m)$ is positive definite \cite{torgerson1958theory,RePEc1938}.
On the other hand, sum of multiple positive definite matrices (i.e. $\mathbf W^{cMM}$) is a positive definite  matrix too \cite{Horn:MA:5509}. Thus, the  eigenvalues of $\mathbf W^{cMM}$ are non-negative and it induces an  $\mathcal{L}_2^2$ embedding.
\end{proof}

Thereby, instead of summing the $\mathbf D^{MM}(m)$ matrices, we sum up the $\mathbf W^{MM}(m)$'s and then perform eigenvalue decomposition, to compute an embedding wherein the pairwise squared Euclidean distances correspond to the collective pairwise Minimax distances. As will be mentioned in the experiments, we perform the above computations on the squared pairwise base dissimilarities .

\textbf{Efficient calculation of dimension-specific Minimax distances.}
In this paper, we particularly study the use of collective Minimax embedding for high-dimensional data (called the \emph{dimension-specific} variant). The  dimension-specific variant may require computing pairwise Minimax distances for each dimension, which can be computationally expensive. However, in this setting, the objects stay in an one-dimensional space.
A main property of one-dimensional data is that sorting them immediately gives a minimum spanning tree. We, then, compute the pairwise distances for each pair of consecutive objects in the sorted list to obtain the edge weights of the minimum spanning tree. Finally, we compute the pairwise Minimax distances from the minimum spanning tree.

\section{One-To-All Minimax Distance Measures}\label{sec:knn-minimax}

In this section, we study computing Minimax distances from a fixed (test) object which can be in particular used in $K$-nearest neighbor search. In this setup, we are given a graph $\mathcal G(\mathbf O,\mathbf D)$ (as the training dataset), and a new object (node) $v$. $\mathbf O$ indicates a set of $N$ (training) objects with the  corresponding pairwise dissimilarities $\mathbf D$, i.e., $\mathbf D$ denotes the weights of the edges in $\mathcal G$.
The goal is to compute the $K$ objects in $\mathbf O$ whose Minimax distances from $v$ are not larger than any other object in $\mathbf O$. We assume there is a function $d(v,\mathbf O)$ which computes the pairwise dissimilarities between $v$ and all the objects in $\mathbf O$. Thereby, by adding $v$ to $\mathcal G$ we obtain the graph $\mathcal G^+(\mathbf O \cup v, \mathbf D\cup d(v,\mathbf O))$.

\subsection{Minimax $K$-nearest neighbor search}

We introduce an incremental approach for computing the Minimax $K$-nearest neighbors of the new object $v$, i.e. we obtain the $(T+1)^{th}$ Minimax neighbor of $v$ given the first $T$ Minimax neighbors. Thereby, first, we define the \emph{partial neighbors} of $v$ as the first Minimax neighbors of $v$ over graph $\mathcal G^+(\mathbf O \cup v, \mathbf D\cup d(v,\mathbf O))$, i.e.,

\begin{eqnarray}
	\mathcal {NP}_T(v) &=& \biggl \{ \{i\} \subseteq \mathbf O: |\{i\}| = T\;,
	 \nexists j \in \{\mathbf O\setminus \mathcal {NP}_T(v)\}: \mathbf D^{MM}_{vj} < \mathbf D^{MM}_{vi}\biggr\}.
\end{eqnarray}

Theorem~\ref{theorem:NextPathNN} provides a way to extend $\mathcal {NP}_T(v)$ step by step until it contains the $K$ nearest neighbors of $v$ according to  Minimax distances.

\begin{theorem}
Given the (traning) graph $\mathcal G(\mathbf O, \mathbf D)$ and a test node $v$, assume we have already computed the first $T$ Minimax neighbors of $v$ (i.e., the set $\mathcal {NP}_T(v)$). Then, the node with the minimal distance to the set $\{v\cup\mathcal {NP}_T(v)\}$ gives the $(T+1)^{th}$ Minimax nearest neighbor of $v$.\footnote{Given a set $S \subset \mathbf O$, the distance of object $i \in \{\mathbf O \setminus S\} $ to $S$ is obtained by the minimum of the distances between $i$ and the objects in $S$.}
\label{theorem:NextPathNN}
\end{theorem}
\begin{proof}
Let us call this new (potential) neighbor $u$. Let $e_{u*}$ indicate the edge (with the smallest weight) connecting $u$ to $\{v \cup\mathcal{NP}_T(v)\}$ and $p \in \{v \cup \mathcal{NP}_T(v)\}$ denote the node at the other side of $e_{u*}$.
We consider two cases:

\begin{enumerate}[leftmargin=*]
\item $\mathbf D_{p,u} \le \mathbf D^{MM}_{v,p}$, i.e. the weight of $e_{u*}$ is not larger than the Minimax distance between $v$ and $p$. Then, $\mathbf D^{MM}_{v,u} = \mathbf D^{MM}_{v,p}$ and therefore $u$ is naturally a \emph{valid} choice.

\item $\mathbf D_{p,u} > \mathbf D^{MM}_{v,p}$, then we show that there exist no other unselected node $u'$ (i.e. from $\{\mathbf O \setminus \mathcal{NP}_T(v)\}$)  which has a smaller Minimax distance to $v$ than $u$.
Thereby:

(a) $D^{MM}_{v,u} > \mathbf D^{MM}_{v,p}$ (according to the assumption $\mathbf D_{p,u} > \mathbf D^{MM}_{v,p}$). Moreover, it can be shown that there is no other path from $v$ to $u$ whose maximum weight is smaller than $\mathbf D_{p,u}$. Because, if such a path exists, then there must be a node $p' \in \mathcal{NP}_T(v)$ which has a smaller distance to an external node like $u'$. This leads to a contradiction since $u$ is the closest neighbor of $\mathcal{NP}_T(v)$.

(b) For any other unselected node $u'$ we  have  $\mathbf D^{MM}_{v,u'} \ge \mathbf D_{p,u}$. Because computing the Minimax distance between $v$ and $u'$ requires visiting an edge whose one side is inside $\mathcal{NP}_T(v)$, but the other side is outside $\mathcal{NP}_T(v)$. Among such edges, $e_{u*}$ has the minimal weight, therefore, $\mathbf D^{MM}_{v,u'}$ cannot be smaller than $\mathbf D_{p,u}$.

Finally, from (a) and (b) we conclude that $\mathbf D^{MM}_{v,u} \le \mathbf D^{MM}_{v,u'}$ for any unselected node $u' \neq u$. Hence, $u$ has the smallest Minimax distance to $v$ among all unselected nodes.
\end{enumerate}
\end{proof}

Theorem~\ref{theorem:NextPathNN} proposes a dynamic programming approach to compute the Minimax $K$-nearest neighbors of $v$.
Iteratively, at each step, we obtain the next Minimax nearest neighbor by selecting the external (unselected) node $u$ which has the minimum distance to the nodes in $\{v\cup\mathcal {NP}_T(v)\}$. This procedure is repeated for $K$ times.  Algorithm~\ref{alg:PathNN} describes the steps in detail. Since the algorithm performs based on finding the nearest unselected node, therefore vector $dist$ is used to keep the minimum distance of unselected nodes to (one of the members of) $\{v \cup\mathcal{NP}_T(v)\}$. Thereby, the algorithm finds a new neighbor in two steps: i) \emph{extension}: it picks the minimum of $dist$ and adds the respective node to $\mathcal{NP}_T(v)$, and ii) \emph{update}: it updates $dist$ by checking if an unselected node has a smaller distance to the new $\mathcal{NP}_T(v)$. Thus $dist$ is updated by $dist = min(dist, \mathbf D_{min\_ind,:})$, except for the members of $\mathcal{NP}_T(v)$.

\begin{algorithm}[!ht]
\caption{Calculation of Minimax $K$-nearest neighbor}
\label{alg:PathNN}
\begin{algorithmic} [1]
\REQUIRE {Graph $\mathcal{G^+}(\mathbf O\cup v, \mathbf D\cup d(v,\mathbf O))$ including a test object $v$.}
\ENSURE A list of $K$ Minimax nearest neighbors of $v$ stored in $\mathcal{NP}(v)$.
\vspace{2mm}

\STATE Initialize vector \emph{dist} by the distance of $v$ to each object in $\mathbf O$.
\STATE $\mathcal{NP}(v) = []$

\FOR{$T=1$ \textbf{to} $K$}
	\STATE\COMMENT{extension}
	\STATE $min\_ind = \arg \min(dist)$
	\STATE $\mathcal{NP}(v).append(min\_ind)$

	\STATE\COMMENT{update}
	\STATE  $dist_{min\_ind} = inf$
	\FOR {$i \notin \mathcal{NP}(v)$}
		\IF {$\mathbf D_{new\_ind,i} < dist_i$}
			\STATE $dist_i =  \mathbf D_{min\_ind,i}$
		\ENDIF
	\ENDFOR
	
\ENDFOR

\STATE \textbf{return} $\mathcal{NP}(v)$
\end{algorithmic}
\end{algorithm}

\paragraph{Computational complexity.}
Each step of Algorithm~\ref{alg:PathNN}, either the extension or the update, requires an $\mathcal O(N)$ running time. Therefore the total complexity is $\mathcal O(N)$ which is the same as for the standard $K$-nearest neighbor method.
The standard $K$-nearest neighbor algorithm computes only the first step. Thus our algorithm only adds a second update step, thereby it is more efficient than the message passing method~\cite{KimC07icml} that requires more visits of the objects and also builds a complete minimum spanning tree in advance on $\mathcal G(\mathbf O,\mathbf D)$.

\subsection{Minimax $K$-NN search, the Prim's algorithm and computational optimality}

According to Theorem \ref{theorem:PathGraphMST}, computing pairwise Minimax distances requires first computing a minimum spanning tree over the underlying graph. Thereby, here, we study the connection between Algorithm~\ref{alg:PathNN} (i.e., Minimax $K$-NN search) and minimum spanning trees. For this purpose, we first consider a general framework for constructing minimum spanning trees called \emph{generalized greedy algorithm} \cite{Gabow1986}.
Consider a forest (collection) of  trees $\left\{T_p\right\}$. The distance between the two trees $T_p$ and $T_q$ is obtained by
\begin{equation}
\Delta T_{pq} = \min_{i \in T_p} \min_{j \in T_q} \mathbf D_{ij} \,.
\end{equation}

The \emph{nearest neighbor tree} of  $T_p$, i.e.  $T_{p}^*$, is obtained by
\begin{equation}
T_{p}^* =\arg \min_{T_q}\Delta T_{pq}\; , \;\; q\ne p \,.
\end{equation}

Then, the edge $e_p^*$ represents  the nearest tree from $T_{p}$.
It can be shown that $e_p^*$ will belong to a minimum spanning tree on the graph as otherwise it yields a contradiction \cite{Gabow1986}.
This result provides a \emph{generic} way to compute a minimum spanning tree. A greedy MST algorithm at each step, i) picks two candidate (base) trees where one is the nearest neighbor of the other, and ii)  combines them via their shortest distance (edge) to build a larger tree. 

This analysis guarantees that Algorithm \ref{alg:PathNN} yields a minimum spanning subtree on $\{v \cup\mathcal{NP}_T(v)\}$ which is additionally a subset of a larger minimum spanning tree on the whole graph $\mathcal G^+(\mathbf O\cup v, \mathbf D\cup d(v,\mathbf O))$. In the context of \emph{generalized greedy algorithm}, Algorithm~\ref{alg:PathNN} generates the MST via growing only the tree started from $v$ and the other candidate trees are singleton nodes.
A complete minimum spanning tree would be constructed, if the attachment continues for $N$ steps, instead of $K$. This procedure, then, would yield a \emph{Prim} minimum spanning tree \cite{Prim1957}. This analysis reveals an interesting property of the \emph{Prim}'s algorithm:

\emph{The Prim's algorithm `sorts' the nodes based on their Minimax distances from/to the initial (test) node $v$}.

\paragraph{Computational optimality.} This analysis also reveals the `\emph{computational optimality}' of our approach (Algorithm~\ref{alg:PathNN}) compared to the alternative Minimax search methods, e.g. the method introduced in~\cite{KimC07icml}. Our algorithm \emph{always} expands the tree which contains the initial node $v$, whereas the alternative methods might sometimes combine the trees that do not have any impact on the Minimax $K$-nearest neighbors of $v$. In particular, the method in~\cite{KimC07icml} constructs a \emph{complete} MST, whereas we build a \emph{partial} MST with only and exactly $K$ edges. Any new node that we add to the partial MST, belongs to the $K$ Minimax nearest neighbors of $v$, i.e., we  do not investigate/add any unnecessary node to that.

\subsection{One-to-all Minimax distances and minimum spanning trees}

A straightforward generalization of Algorithm~\ref{alg:PathNN} gives sorted \emph{one-to-all} Minimax distances from the target (test) $v$ to the all other nodes. We only need to run the steps (extension and update) for $N$ times instead of $K$.
In Theorem \ref{theorem:PathGraphMST}, we observed that computing \emph{all-pair} Minimax distances is  consistent/equivalent with computing a minimum spanning tress, such that a MST provide all the necessary and sufficient edge weights for this purpose. Then later we have proposed an efficient algorithm for the \emph{one-to-all} problem that also yields computing a (partial) minimum spanning tree (according to Prim's algorithm). Therefore, here we study the following question:

\emph{Is it necessary to build a minimum spanning tree to compute the `one-to-all' Minimax distances, similar to `all-pair' Minimax distances?}

We reconsider the proof of Theorem~\ref{theorem:NextPathNN} which has led us to Algorithm~\ref{alg:PathNN}. The proof investigates two cases: i) $\mathbf D_{p,u} \le \mathbf D^{MM}_{v,p}$, ii) $\mathbf D_{p,u} > \mathbf D^{MM}_{v,p}$.  We investigate the first case in more detail.

Let us look at the largest Minimax distance selected up to step $T$, i.e., the edge with maximal weight whose both sides occur inside $\{v \cup\mathcal{NP}_T(v)\}$ and its weight represents a Minimax distance. We call this edge $e^{max}_T$.
Among different external neighbors of $\{v \cup\mathcal{NP}_T(v)\}$, any node $u' \in \{\mathbf O \setminus \mathcal{NP}_T(v)\}$ whose minimal distance to $\{v \cup\mathcal{NP}_T(v)\}$ does not exceed the weight of $e^{max}_T$ can be selected, even if it does not belong to a minimum spanning tree over $\mathcal G^+$. Because, anyway the Minimax distance will be still the weight of $e^{max}_T$.

In general, any node $u' \in \{\mathbf O \setminus \mathcal{NP}_T(v)\}$, whose Minimax distance to a selected member $p$, i.e. $\mathbf D^{MM}_{p,u}$,  is not larger than the weight of $e^{max}_T$, can be selected as the next nearest neighbor\footnote{We notice that in the second case, i.e. when $\mathbf D_{p,u} > \mathbf D^{MM}_{v,p}$, we need to add an unselected node whose Minimax to $v$ (to one of the nodes in $\mathcal{NP}_T(v)$) is equal to $\mathbf D_{p,u}$. This situation might also yield adding an edge which does not necessarily belong to  a MST over $\mathcal G^+$. The argument is similar to the case when $\mathbf D_{p,u} \le \mathbf D^{MM}_{v,p}$.}.
This is concluded from the following property of Minimax distances:
\begin{equation}
  \mathbf D^{MM}_{v,u}  \le  \max(\mathbf D^{MM}_{v,p},\mathbf D^{MM}_{p,u}) \, ,
\end{equation}
where $p$ is an arbitrary object (node).
In this setting we then have, \\
1. $\mathbf D^{MM}_{v,u}  \le  \max(\mathbf D^{MM}_{v,p},\mathbf D^{MM}_{p,u})$, and \\
2. $\mathbf D^{MM}_{v,p}>\mathbf D^{MM}_{p,u}$.
\\
Thus, we conclude  $\mathbf D^{MM}_{v,u}  \le  \max(\mathbf D^{MM}_{v,p})$, i.e. $\mathbf D^{MM}_{v,u}  =  \mathbf D^{MM}_{v,p}$.

An example is shown in Figure~\ref{fig:Analysis}, where $K$ is fixed at $2$. After computing the first nearest neighbor (i.e., $p$), the next one can be any of the remaining objects, as their Minimax distance to $v$ is the weight of the edge connecting $p$ to $v$ (Figure~\ref{fig:Analysis1}) or $p$ to $u$ (Figure~\ref{fig:Analysis2}). Thereby, one could choose a node such that the \emph{partial} MST on $\{v \cup\mathcal{NP}(v)\}$ is not a subset of any \emph{complete} MST on the whole graph.
Therefore, computing \emph{one-to-all} Minimax distances might not necessarily require taking into account  construction of a minimum spanning tree.

Thereby, such an analysis suggests an even more generic algorithm for computing  \emph{one-to-all} Minimax distances (including Minimax $K$-nearest neighbor search) which dose not necessarily yield a MST on the entire graph. To expand $\mathcal{NP}_T(v)$, we add a new $u'$ whose \emph{Minimax distance} (not the direct dissimilarity) to $\{v \cup\mathcal{NP}_T(v)\}$ is minimal.
Algorithm~\ref{alg:PathNN} is only one way, but an efficient way, that even sorts one-to-all Minimax distances.

\begin{figure}[t!]
    \centering
    \subfigure[$\mathbf D_{p,u} \le \mathbf D^{MM}_{v,p}$]
    {
        \includegraphics[width=0.184\textwidth]{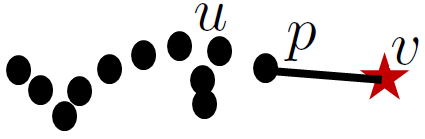}
        \label{fig:Analysis1}
    }
    \hspace{10mm}
    \subfigure[$\mathbf D_{p,u} > \mathbf D^{MM}_{v,p}$]
    {
        \includegraphics[width=0.16\textwidth]{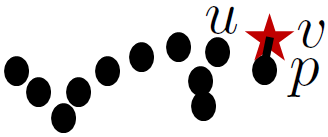}
        \label{fig:Analysis2}
    }
    \caption{Computing Minimax nearest neighbors (of the \emph{red} object) does not necessarily need to agree with a complete MST on $\mathcal G^+$, as any remaining node can be replaced with $u$.}
    \label{fig:Analysis}
\end{figure}

\subsection{Outlier detection with Minimax $K$-nearest neighbor search}

While performing $K$-nearest neighbor search, the test objects might not necessarily comply with the structure in the train dataset, i.e. some of the test objects might be outliers or belong to other classes than those existing in the train dataset. More concretely, when computing the $K$ nearest neighbors of $v$, we might realize that $v$ is an outlier or irrelevant with respect to the underlying structure in $\mathbf O$. We study this case in more detail and propose an efficient algorithm which while computing Minimax $K$ nearest neighbors, it also detects whether $v$ could be an outlier.  Notice that the ability to detect the outliers will be an additional feature of our algorithm, i.e., we do not aim at proposing the best outlier detection method, rather, the goal is to maximize the benefit from performing Minimax $K$-NN while still having an $\mathcal O(N)$ runtime.

Thereby, we follow our analysis by investigating another special aspect of Algorithm~\ref{alg:PathNN}: as we extend $\mathcal{NP}_T(v)$, the edge representing the Minimax distance between $v$ and the new member $u$ (i.e., the edge with the largest weight on the path indicating the Minimax distance between $v$ and $u$) is always connected to $v$.
For this purpose, we reconsider the proof of Theorem~\ref{theorem:NextPathNN}. When $u$ is being added to $\mathcal{NP}_T(v)$, two special cases might occur:
\begin{enumerate}
\item $u$ is directly connected to $v$, i.e. $p=v$.
\item $u$ is not directly connected to $v$ (i.e. $p\ne v$), but its Minimax distance to $v$ is \emph{not} represented by $\mathbf D_{p,u}$, i.e., $\mathbf D_{p,u} < \mathbf D^{MM}_{v,p}$.
\end{enumerate}

Figure~\ref{fig:OutlierDemonstration} illustrates these situations, where $K=4$. Among the nearest neighbors, two are directly connected to $v$ (i.e., $p=v$), and the two others are connected to $v$ via the early members of $\mathcal{NP}_T(v)$, but the Minimax distance is still represented by the edges connected to $v$. In other words, there is no edge in graph $\mathcal G(\mathbf O,\mathbf D)$ which represents a Minimax distance, although some edges of $\mathcal G(\mathbf O,\mathbf D)$ are involved in Minimax paths. Thus, the type of connection between $v$ and its neighbors is \emph{different} from the type of the distances inside graph $\mathcal G(\mathbf O,\mathbf D)$ (without test object $v$). Thereby, in this case we report $v$ as an \emph{outlier}. Notice that both of the above mentioned conditions should occur, i.e. if we have always $p=v$ (as shown in Figure~\ref{fig:OutlierNOT1Demonstration} where the nearest neighbors are always connected to $v$), then $v$ might not be labeled as an \emph{outlier}.

On the other hand, an shown in Figure~\ref{fig:OutlierNOT2Demonstration}, if some edges of $\mathcal G(\mathbf O,\mathbf D)$ contribute in computing the Minimax $K$-nearest neighbors of $v$ (in addition to the edges which meet $v$), then $v$ has the same type of neighborhood as some  other nodes in $\mathcal G$. Thereby, it is \emph{not} labeled as an outlier.

This analysis suggests an algorithm for simultaneously computing Minimax $K$-nearest neighbors and detecting if $v$ is an outlier (Algorithm~\ref{alg:PathOutlier}).
For this purpose, we use a vector called \emph{updated}, which determines the type of nearest neighborhood of each external object to $\{v \cup\mathcal{NP}_T(v)\}$, i.e.,
\begin{enumerate}
\item $updated_i=0$, if $v$ is the node representing the nearest neighbor of $i$ in the set  $\{v \cup\mathcal{NP}_T(v)\}$.
\item  $updated_i=1$, otherwise.
\end{enumerate}

At the beginning, $\mathcal{NP}_T(v)$ is empty. Thereby, \emph{updated} is initialized by a vector of zeros. At each update step, whenever we modify an element of the vector \emph{dist}, we then set the corresponding index in \emph{updated} to $1$. At each step of Algorithm~\ref{alg:PathNN}, a new edge is added to the set of selected edges whose weight is $\mathbf D_{p,u}$.
As mentioned before, $v$ is labeled as an outlier if, i) some of the edges are directly connected to $v$, and some others indirectly, and ii) no indirect edge weight represents a Minimax distance, i.e. the  minimum of the edge weights directly connected to $v$ (stored in \emph{min\_direct}) is \emph{larger} than the maximum of the edge weights not connected to $v$ (stored in \emph{max\_indrct}).
The later corresponds to the condition that  $\mathbf D^{MM}_{v,p} > \mathbf D_{p,u}$. Thereby, whenever  we pick the nearest neighbor of  $\{v \cup\mathcal{NP}_T(v)\}$ in the extension step,  we then check the \emph{updated} status of the new member:
\begin{enumerate}
\item If $updated_{min\_ind}=1$, we then update \emph{max\_indrct} by $\max(max\_indrct,dist_{min\_ind})$.
\item If $updated_{min\_ind}=0$, we then update  \emph{min\_direct} by $\min(min\_direct,dist_{min\_ind})$.
\end{enumerate}

Finally, if $min\_direct > max\_indrct$ and $max\_indrct \ne -1$ ($max\_indrct$ is initialized by $-1$; this condition ensures that at least one indirect edge has been selected), then $v$ is labeled as an outlier. Algorithm~\ref{alg:PathOutlier} describes the procedure in detail. Compared to Algorithm~\ref{alg:PathNN}, we have added: i) steps 11-15 for keeping the statistics about the direct and indirect edge weights (i.e., via updating \emph{max\_indrct} and \emph{min\_direct}), ii) an additional step in line 20 for updating the type of the connection of the external nodes to the set $\{v\cup \mathcal{NP}_T(v) \}$, and iii) finally checking whether $v$ is an outlier\footnote{Our goal is not to propose the best outlier detection algorithm, rather, we aim at adding a new feature (i.e., detecting potential outliers) to Minimax $K$-nearest neighbor search while we require the runtime to be still $\mathcal O(N)$. Thus, we improve the gain of performing a linear-time Minimax $K$-NN search query.}.
\paragraph{Computational complexity.} The extension step is done in $\mathcal O(N)$ via a sequential search. Lines 11-15 are constant in time and the loop 17-22 requires an $\mathcal O(N)$ time. Thus, the total runtime is $\mathcal O(N)$ which is identical to standard $K$-nearest neighbor search over arbitrary graphs.

\begin{figure*}[htb]
    \centering
    \subfigure[outlier]
    {
        \includegraphics[width=0.25\textwidth]{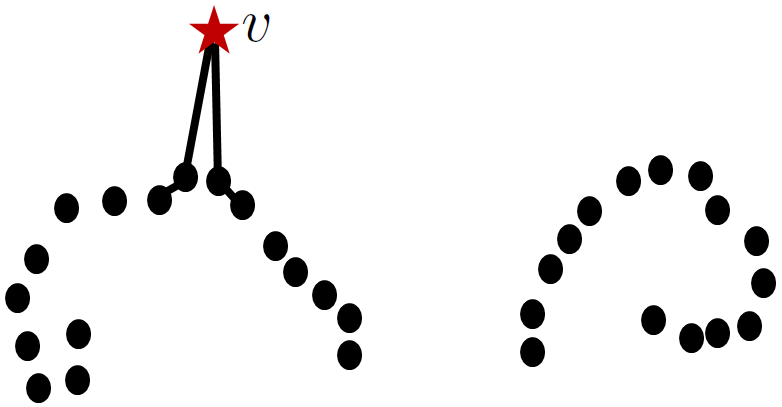}
        \label{fig:OutlierDemonstration}
    }
    \unskip\ \vrule\
    \subfigure[not outlier]
    {
        \includegraphics[width=0.25\textwidth]{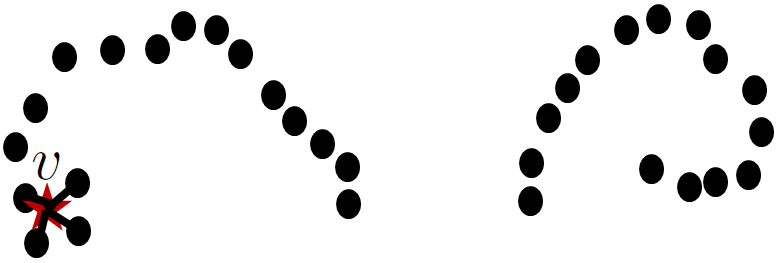}
        \label{fig:OutlierNOT1Demonstration}
    }
    \unskip\ \vrule\
    \subfigure[not outlier]
    {
        \includegraphics[width=0.25\textwidth]{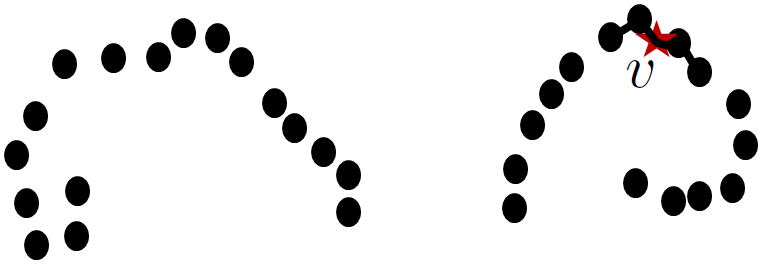}
        \label{fig:OutlierNOT2Demonstration}
    }
    \caption{The edges representing the Minimax distances can provide useful information for detecting outliers.}
    \label{fig:OutlierOrNot}
\end{figure*}

\begin{algorithm}[!ht]
\caption{Outlier detection}
\label{alg:PathOutlier}
\begin{algorithmic} [1]
\REQUIRE {Graph $\mathcal{G^+}(\mathbf O\cup v, \mathbf D\cup d(v,\mathbf O))$ including a test object $v$.}
\ENSURE Minimax $K$ nearest neighbors of $v$ stored in $\mathcal{NP}(v)$, detecting whether $v$ is an outlier.
\vspace{2mm}

\STATE $\mathcal{NP}(v) = []$
\STATE Initialize vector \emph{dist} by the distance of $v$ to each object in $\mathbf O$.
\STATE Initialize vector \emph{updated} by $0$.
\STATE $max\_indrct = -1$
\STATE $min\_direct = inf$

\FOR{$T=1$ \textbf{to} $K$}
	\STATE\COMMENT{extension}
	\STATE $min\_ind = \arg \min(dist)$
	\STATE $\mathcal{NP}(v).append(min\_ind)$

	\STATE\COMMENT{update}	
	\IF {$updated_{min\_ind}=1$}
		\STATE \emph{max\_indrct} $=\max(max\_indrct,dist_{min\_ind})$
	\ELSE
		\STATE \emph{min\_direct} $=\min(min\_direct,dist_{min\_ind})$
	\ENDIF
	
	\STATE  $dist_{min\_ind} = inf$
	\FOR {$i \notin \mathcal{NP}(v)$}
		\IF {$\mathbf D_{min\_ind,i} < dist_i$}
			\STATE $dist_i =  \mathbf D_{min\_ind,i}$
			\STATE $updated_i = 1$
		\ENDIF
	\ENDFOR
	
\ENDFOR

\STATE $is\_outlier = 0$
\IF{$max\_indrct\neq -1$ \textbf{and} $min\_direct>max\_indrct$}
	\STATE $is\_outlier = 1$
\ENDIF

\STATE \textbf{return} $\mathcal{NP}(v)$, \emph{is\_outlier}
\end{algorithmic}
\end{algorithm}

\section{Experiments}\label{sec:experiments}
We experimentally study the performance of Minimax distance measures on a variety of synthetic and real-world datasets and illustrate how the use of Minimax distances as an intermediate step improves the results. In each dataset, each object is represented by a vector.\footnote{When not mentioned, we compute the pairwise squared Euclidean distances between the vectors to construct the base distance matrix $\mathbf D$.}

\subsection{Classification with Minimax distances}
First, we study classification with Minimax distances.
We use Logistic Regression (LogReg) and Support Vector Machines (SVM) as the baseline methods and investigate how performing these methods on the vectors induced from Minimax distances improves the accuracy of prediction.
With SVM, we examine three different kernels: i. linear (lin), ii. radial basis function (rbf), and iii. sigmoid (sig), and choose the best result. With Minimax distances, we only use the linear kernel, since we assume that Minimax distances must be able to capture the correct classes, such that they can be then discriminated via a linear separator.

\noindent\textbf{Experiments with synthetic data}

We first perform our experiments on two synthetic datasets, called: i) DS1 \cite{Chang:2008:RPS}, and ii) DS2~\cite{Veenman02amaximum},  which are shown in Figure~\ref{fig:Datasets_synthetic}. The goal is to demonstrate the superior ability of Minimax distances to capture the correct class-specific structures, particularly when the classes have different types (DS2), compared to kernel methods.
Table~\ref{table:accuracy_synthetic} shows the accuracy scores for different methods. The standard SVM is performed with three different kernels (lin, rbf and sig), and the best choice which is the rbf kernel is shown. As mentioned, with Minimax distances, we only use the linear kernel.
We observe that performing classification on Minimax vectors yields the best results, since it enables the method to better identify the correct classes. The datasets differ in the type and consistency of the classes. DS1 contains very similar classes which are Gaussian. But DS2 consists of classes which differ in shape and type. Therefore, for DS1 we are able to find an optimal kernel (rbf, since the classes are Gaussian) with a \emph{global} form and parametrization, which fits with the data and thus yields very good results. However, in the case of DS2, since classes have different shapes, then a single  kernel \emph{is not} able to capture correctly all of them. For this dataset, LogReg and SVM with Minimax vectors perform better, since they enable to adapt to the class-specific structures.  Note that in the case of DS1, using Minimax vectors is equally good to using the optimal rbf kernel. Remember that, to Minimax vectors, we apply only SVM with a linear kernel. Figures~\ref{fig:DS1_plot} and~\ref{fig:DS2_plot} show the accuracy and eigenvalues w.r.t. different dimensionality of Minimax vectors. As mentioned earlier, the dimensions of Minimax vectors are sorted according to the respective eigenvalues, since a larger eigenvalue indicates a higher importance. By choosing only few dimensions, the accuracy attains its maximal value. We will elaborate in more detail on this further on.

We note that after computing the matrix of pairwise Minimax distances, one can in principle apply any of the embedding methods (e.g., PCA, SVD, etc) either to the original pairwise distance matrix or to the pairwise Minimax distance matrix. Therefore, our contribution is orthogonal to such embedding methods. For example on these  synthetic datasets, when we apply PCA to the original distance matrix, we do not observe a significant improvement in the results (e.g., $0.7838$  on DS1 and  $0.6671$  on DS2  with LogReg).

 \begin{figure*}[ht!]
    \centering
    \subfigure[DS1]
    {
        \includegraphics[width=0.28\textwidth]{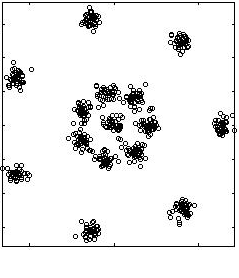}
        \label{fig:Datasets_DS1}
    }
    \hspace{5mm}
    \subfigure[DS2]
    {
        \includegraphics[width=0.28\textwidth]{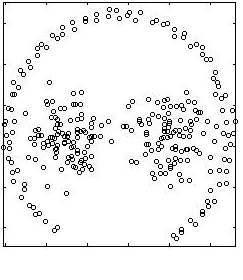}
        \label{fig:Datasets_DS2}
    }
\\
    \subfigure[results on DS1]
    {
        \includegraphics[width=0.33\textwidth]{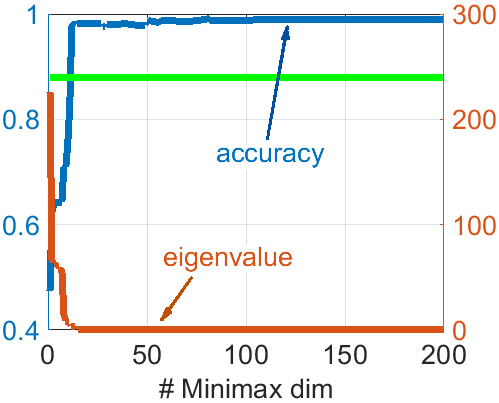}
        \label{fig:DS1_plot}
    }
    \hspace{5mm}
    \subfigure[results on DS2]
    {
        \includegraphics[width=0.33\textwidth]{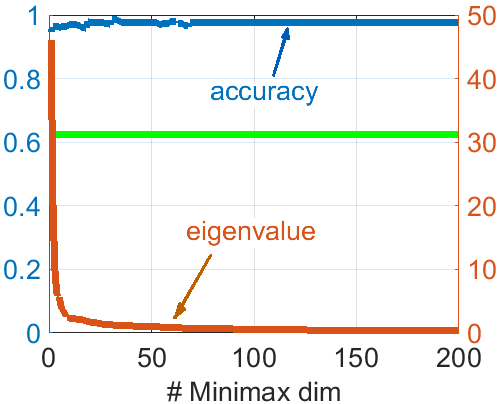}
        \label{fig:DS2_plot}
    }
    \caption{Illustration of DS1, DS2 and the accuracy.
    Different classes of DS1 have a similar structure, but the classes are different in DS2.
    The accuracy scores (shown for LogReg-MM) are stable w.r.t. the dimensionality of the Minimax vectors (Figures~\ref{fig:DS1_plot} and~\ref{fig:DS2_plot}). The straight green line shows the accuracy for the base LogReg.}
    \label{fig:Datasets_synthetic}
\end{figure*}

\begin{table*}[ht!]
\caption{Accuracy of different methods on synthetic datasets. Minimax measure particularly improves the results when the classes in the dataset have different shapes and types (e.g. DS2).}
\centering 
\begin{tabular}{|| c || c c c c || c c ||} 
\hline\hline 
 &\multicolumn{4}{c||}{standard}&\multicolumn{2}{c||}{Minimax}\\
 \hline
 dataset & SVM-lin& SVM-rbf & SVM-sig & LogReg & SVM-lin & LogReg \\ [0.5ex] 
\hline 
DS1 & 0.5749 & 0.9924 & 0.4260 & 0.8066 & 0.9917 & 0.9918  \\  
\hline 
DS2 & 0.5654 & 0.9295 & 0.3294 & 0.6252 & 0.9950 & 0.9983  \\  
\hline 
\end{tabular}
\label{table:accuracy_synthetic} 
\end{table*}

\noindent\textbf{Classification of UCI data}

We then perform  real-world experiments on twelve datasets from different domains,  selected from the UCI repository~\cite{Lichman:2013}:\footnote{The specifications of the datasets can be found at \url{https://archive.ics.uci.edu/ml/index.php}} \\
(1) \emph{Balance Scale}: contains $625$ observations modeling $3$ types of psychological experiments. \\
(2) \emph{Banknote Authentication}: includes $1372$ images taken from genuine and forged banknote-like specimens (number of classes is $2$).\\
(3) \emph{Cloud}: consists of $1024$  $10$-dimensional vectors, each dimension representing a specific parameter.\\
(4) \emph{Contraceptive Method}: contains information of $1473$ women, where the three classes are about the pregnancy status.\\
(5) \emph{Glass Identification}: contains $6$ types (classes) of glass w.r.t the oxide content. The number of instances is $214$.\\
(6) \emph{Haberman Survival}: contains  the survival of $306$ patients who had surgery for breast cancer. The number of classes is $2$.\\
(7) \emph{Hayes Roth}: is about a study on human subjects which contains $160$ instances and $3$ classes. \\
(8) \emph{Ionosphere}: includes $351$ $34$-dimensional instances collected from radars and organized into $2$ classes. \\
(9) \emph{Lung Cancer}: describes $3$ types of pathological lung cancer, including $32$ instances each with $56$ dimensions. \\
(10) \emph{Perfume}: consists of odors of $20$ different perfumes (classes), where the data is collected via OMX-GR sensor. There are in total $560$ measurements. \\
(11) \emph{Skin Segmentation}: the original dataset contains $245,057$ instances generated using skin textures from face images of different people. However, to make the classification task more difficult, we pick only the first $1,000$ instances of each class (to decrease the number of objects per class). The target variable is skin or non-skin sample, i.e. the number of classes is $2$.  \\
(12) \emph{User Knowledge}: describes $403$ students' knowledge level ($4$ classes) about the subject of Electrical DC Machines. \\
We call these datasets respectively UC1, UCI2, ..., UCI12.

\begin{table*}[ht!]
\caption{Accuracy scores of different methods on UCI datasets, when $60\%$ of the data is used for training. Using Minimax vectors improves the results. The best results are bolded.}
\centering 
\begin{tabular}{|| c || c c c c || c c || c c ||} 
\hline\hline 
 &\multicolumn{4}{c||}{standard}&\multicolumn{2}{c||}{Minimax}&\multicolumn{2}{c||}{dim.spec. Minimax}\\
 \hline
 data & SVM-lin& SVM-rbf & SVM-sig & LogReg & SVM-lin & LogReg & SVM-lin & LogReg \\ [0.5ex] 
\hline 
UCI1 & 0.8709 & 0.8974 & 0.4468 & 0.8687 & 0.6187 & 0.6086 & 0.9211 & \bf{0.9739}   \\  
\hline 
UCI2 & 0.9876 & \bf{1.0000} & 0.5577 & 0.9872 & 0.9989 & \bf{1.0000} & 0.8847 & 0.9827  \\  
\hline 
UCI3 & 0.9988 & 0.5788 & 0.5349 & 0.9988 & \bf{1.0000} & 1.0000 & \bf{1.0000} & \bf{1.0000}  \\  
\hline 
UCI4 & 0.5190 & 0.5533 & 0.4250 &  0.5107 & 0.5647 & \bf{0.5747} & 0.5392 & 0.5389 \\  
\hline 
UCI5 & 0.5924 & 0.6012 & 0.3371 & 0.6053 & 0.5971 & 0.6671 & 0.4918 & \bf{0.6347}  \\  
\hline 
UCI6 & 0.6893 & 0.7230 & 0.7344 & 0.7426 & \bf{0.7434} & 0.7377 & 0.7418 & 0.7352 \\  
\hline 
UCI7 & 0.6058 & 0.7750 & 0.3596 & 0.5365 & 0.7038 & 0.7115 & \bf{0.8635} & 0.8558  \\  
\hline 
UCI8 & 0.8779 & 0.9300 & 0.6536 & 0.8621 & 0.9457 & 0.9450 & 0.8843 & \bf{0.9336}  \\  
\hline 
UCI9 & 0.6917 & 0.7500 & 0.7500 & 0.6917 & 0.7750 & 0.7500 & \bf{0.8333} & \bf{0.8333}  \\  
\hline 
UCI10 & 0.7783 & 0.9318 & 0.1498 & 0.6933 & 0.9865 & \bf{0.9870} & 0.9798 &  0.9830 \\  
\hline 
UCI11 & 0.9287 & 0.9693 & 0.7635 & 0.8980 & \bf{0.9994} & \bf{0.9994} & 0.9906 & 0.9919  \\  
\hline 
UCI12 & 0.7835 & 0.7233 & 0.3019 & \bf{0.8612} & 0.5893 & 0.6757 &  0.6010 & 0.8592  \\  
\hline 
\end{tabular}
\label{table:accuracy_realworld60} 
\end{table*}

\begin{table*}[htb]
\caption{Accuracy of different methods when only $10\%$ of the data is used for training, where using Minimax distances often improves the  results. The best results are bolded. 
}
\centering 
\begin{tabular}{|| c || c c c c || c c || c c ||} 
\hline\hline 
 &\multicolumn{4}{c||}{standard}&\multicolumn{2}{c||}{Minimax}&\multicolumn{2}{c||}{dim.spec. Minimax}\\
 \hline
 data & SVM-lin& SVM-rbf & SVM-sig & LogReg & SVM-lin & LogReg & SVM-lin & LogReg \\ [0.5ex] 
\hline 
UCI1 & 0.8459 & 0.8477 & 0.4566 & \bf{0.8694} & 0.5114 & 0.6021 & 0.8270 & 0.7879  \\  
\hline 
UCI2 & 0.9859 & 0.9843 & 0.5305 & 0.9870 & 0.9895 & \bf{0.9916} & 0.6632 & 0.9060  \\  
\hline 
UCI3 & 0.9607 & 0.5207 & 0.5069 & 0.9849 & \bf{1.0000} & \bf{1.0000} & \bf{1.0000} & \bf{1.0000} \\  
\hline 
UCI4 & 0.4746 & 0.5078 & 0.4082 & 0.4871 & 0.4761 & \bf{0.5206} & 0.4714 & 0.4655 \\  
\hline 
UCI5 & 0.3417 & 0.4677 & 0.3385 & 0.3958 & 0.4365 & 0.4844 & 0.4100 & \bf{0.5000} \\  
\hline 
UCI6 & 0.7011 & 0.7275 & 0.7347 & 0.7311 & \bf{0.7369} & 0.7362 & 0.7336 & 0.7176  \\  
\hline 
UCI7 & 0.4258 & 0.4593 & 0.3352 & 0.4775 & 0.5186 & 0.5758 & 0.5958 & \bf{0.6636}  \\  
\hline 
UCI8 & 0.8033 & 0.7692 & 0.6256 & 0.7943 & 0.9043 & \bf{0.9097} & 0.8000 & 0.8786  \\  
\hline 
UCI9 & 0.7045 & 0.7227 & 0.7136 & 0.7205 & 0.7159 & 0.6750 & 0.7386 & \bf{0.7500}  \\  
\hline 
UCI10 & 0.6937 & 0.5732 & 0.1053 & 0.5129 & \bf{0.8787} & 0.8775 & 0.8236 & 0.8481  \\  
\hline 
UCI11 & 0.9161 & 0.9294 & 0.6203 & 0.9166 & \bf{0.9983} & \bf{0.9983} & 0.9855 & 0.9861  \\  
\hline 
UCI12 & 0.5179 & 0.3772 & 0.3129 & 0.7200 & 0.3759 & 0.5315 & 0.3681 & \bf{0.7638}  \\  
\hline 
\end{tabular}
\label{table:accuracy_realworld10} 
\end{table*}

\begin{figure*}[ht!]
    \centering
    \hspace{-5mm}
    \subfigure[\emph{Glass Identification}]
    {
        \includegraphics[width=0.33\textwidth]{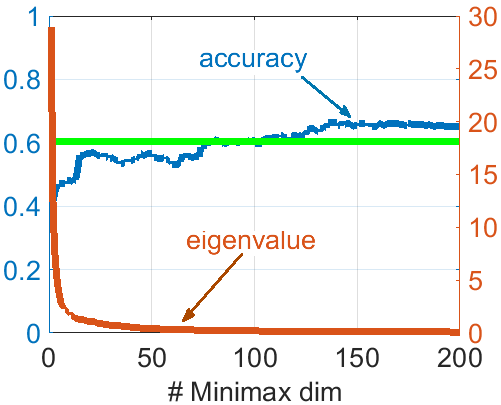}
        \label{fig:Glass_Identification_plot}
    }
    \hspace{5mm}
    \subfigure[\emph{Hayes Roth}]
    {
        \includegraphics[width=0.33\textwidth]{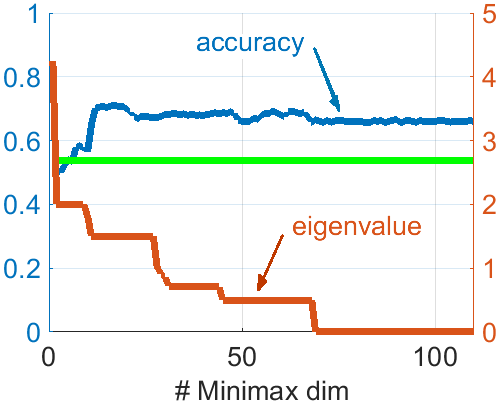}
        \label{fig:Hayes_Roth_plot}
    }
      \\
    \subfigure[\emph{Perfume}]
    {
        \includegraphics[width=0.33\textwidth]{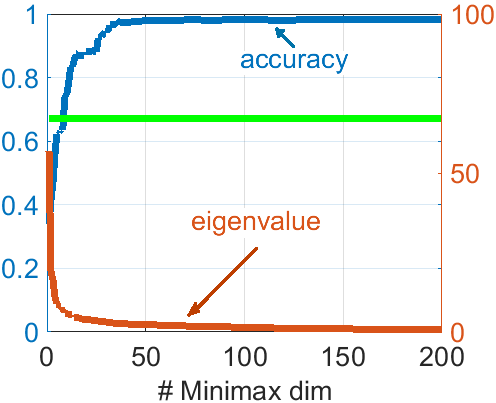}
        \label{fig:Perfume_plot}
    }
    \hspace{5mm}
    \subfigure[\emph{Skin Segmentation}]
    {
        \includegraphics[width=0.33\textwidth]{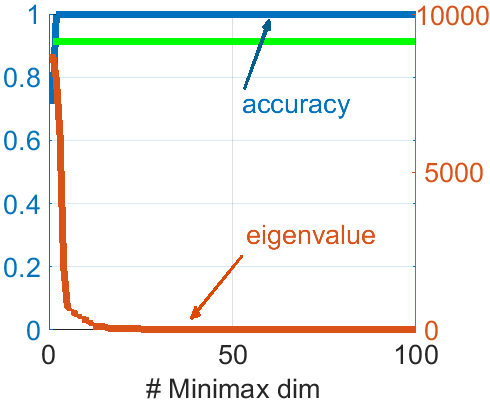}
        \label{fig:Skin_Segmentation_plot}
    }
    \caption{Accuracy score of LogReg-MM applied to different datasets when we choose different number of dimensions of Minimax vectors.
    The straight green lines show the base LogReg result.}
    \label{fig:plots_realworld}
\end{figure*}

\textbf{Accuracy scores.}
Table~\ref{table:accuracy_realworld60} shows the results for different methods applied to the datasets, when $60\%$ of the objects are used for training. We have repeated the random split of the data for $20$ times and report the average results. The scores and the ranking of different methods are very consistent among different splits, such that the standard deviations are low (i.e., maximally about $0.016$).
We observe that often performing the classification methods on Minimax vectors improves the classification accuracy. In only very few cases the standard setup outperforms (slightly). In the rest, either the Minimax vectors or the dimension-specific variant of Minimax vectors yield a better performance. In particular, the Minimax variant is more appropriate for low dimensional data, whereas the dimension-specific Minimax variant outperforms on high-dimensional data. We elaborate more on the choice between them in the `model order selection' section.

Table \ref{table:accuracy_realworld10} shows the results when only $10\%$ of the objects are used for training. For this setting, we observe a consistent performance with the setting where $60\%$ of data is used for training, i.e. the use of Minimax vectors (either the standard variant or the dimension-specific one) improves the accuracy scores. In this setting, only with UCI12 the non-Minimax method performs better. However, even on this dataset, the dimension-specific Minimax performs very closely to the best result.
Such a consistency is observed for other ratios of train and test sets too.

\textbf{Model order selection.}
Choosing the appropriate number of dimensions for Minimax vectors (i.e. their dimensionality) constitutes a model order selection problem. We study in detail how the dimensionality of the Minimax vectors affects the results.  Figure~\ref{fig:plots_realworld} shows the accuracy scores for Minimax-LogReg applied to four of the datasets w.r.t. different number of dimensions (the other datasets behave similarly).
The dimensions are ordered according to their importance (the value of respective eigenvalue). Choosing a very small number of dimensions might be insufficient since we might lose some informative dimensions, which yields underfitting. By increasing the dimensionality, the method extracts more sufficient information from the data, thus the accuracy improves. We note that this phase occurs for a very wide range of choices of dimensions. However, by increasing the number of dimensions even more, we might include non-informative or noisy features (with very small eigenvalues), where then the accuracy stays constant or decreases slightly, due to overfitting. However, an interesting advantage of this approach is that the overfitting dimensions (if there exists any) have a very small eigenvalue, thus their impact is negligible. This analysis leads to a simple and effective model order selection criteria: Fix a small threshold and pick the dimensions whose respective eigenvalues are larger than the threshold. The exact value of this threshold may not play a critical role  or it can be estimated via a validation set in a supervised learning setting.

\textbf{Model selection.}
According to our experimental results, very often, either the standard Minimax classification or the dimension-specific variant outperform the baselines. In a supervised learning setting, the correct choice between these two Minimax variants (i.e., the model selection task) can be made via measuring the prediction ability, e.g. the accuracy score. However, this approach might not be applicable for unsupervised learning, e.g. clustering, where the ground truth is not given. We investigate a simple heuristics which can be employed in any arbitrary setting. When we perform eigen decomposition of the centered matrix ($\mathbf W^{MM}$ or $\mathbf W^{cMM}$), we normalize the eigenvalues by the largest of them such that the largest eigenvalue equals $1$. Then, the variant that yields a shaper decay in eigenvalues is supposed to be potentially a better model.  A  sharper decay of the eigenvalues indicates a tighter confinement to a low dimensional subspace, i.e. lower complexity in data representation. This possibly yields a better learning, and thereby a higher accuracy score.

To analyze this heuristics, we observe that  in our experiments on \emph{Cloud}, \emph{Glass Identification}, \emph{Haberman Survival}, \emph{Perfume} and \emph{Skin Segmentation} either the accuracy scores are not significantly distinguishable or there is no consistent ranking of the two different variants (Minimax or dimension-specific Minimax) such that it is difficult to decide which one is better. For these datasets the eigenspectra are very similar as shown for example for \emph{Cloud} dataset in Figure~\ref{fig:Cloud_eigen}.
Among the remaining datasets, the heuristics works properly on \emph{Balance Scale}, \emph{Contraceptive Method}, \emph{Hayes Roth}, \emph{Lung Cancer} and \emph{User Knowledge} datasets, where two sample eigenspectra are shown in Figures~\ref{fig:Balance_Scale_eigen} and~\ref{fig:Contraceptive_Method_Choice_eigen}. Finally, the heuristics fails only on the two \emph{Banknote Authentication} and \emph{Ionosphere} datasets. However, for these datasets the accuracy scores as well as the decays are rather similar, as shown for example for \emph{Banknote Authentication} in Figure~\ref{fig:Banknote_Authentication_eigen}.

\begin{figure*}[ht!]
    \centering
    \subfigure[\emph{Cloud}]
    {
        \includegraphics[width=0.33\textwidth]{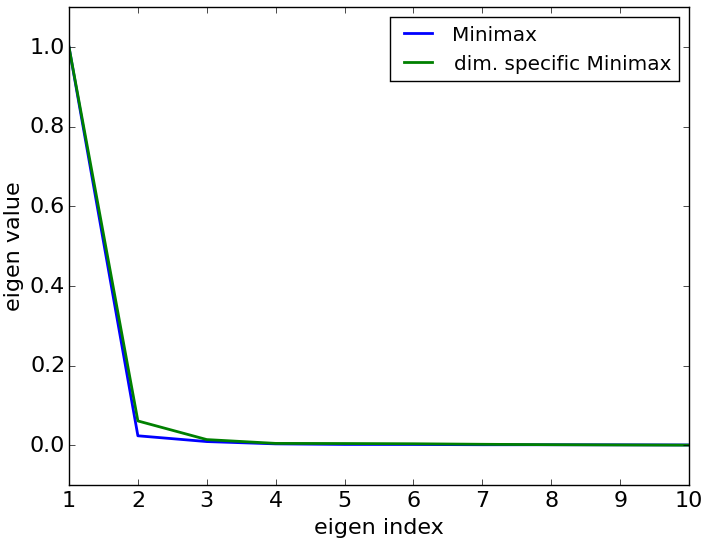}
        \label{fig:Cloud_eigen}
    }
    \hspace{5mm}
    \subfigure[\emph{Balance Scale}]
    {
        \includegraphics[width=0.33\textwidth]{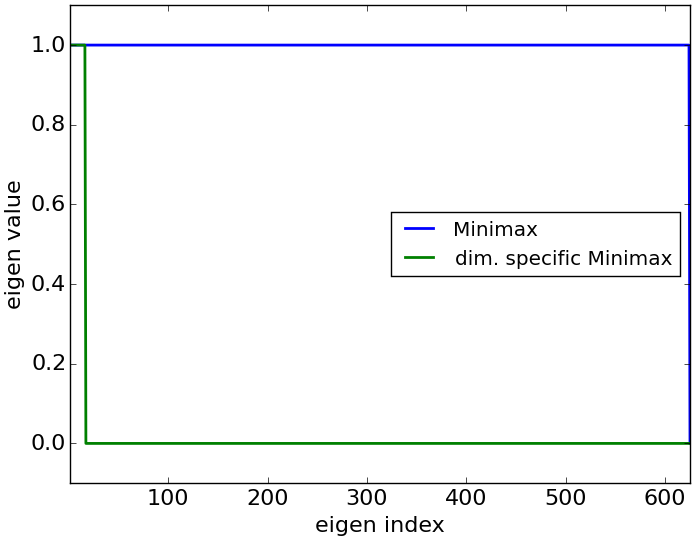}
        \label{fig:Balance_Scale_eigen}
    }
      \\
    \subfigure[\emph{Contraceptive Method}]
    {
        \includegraphics[width=0.33\textwidth]{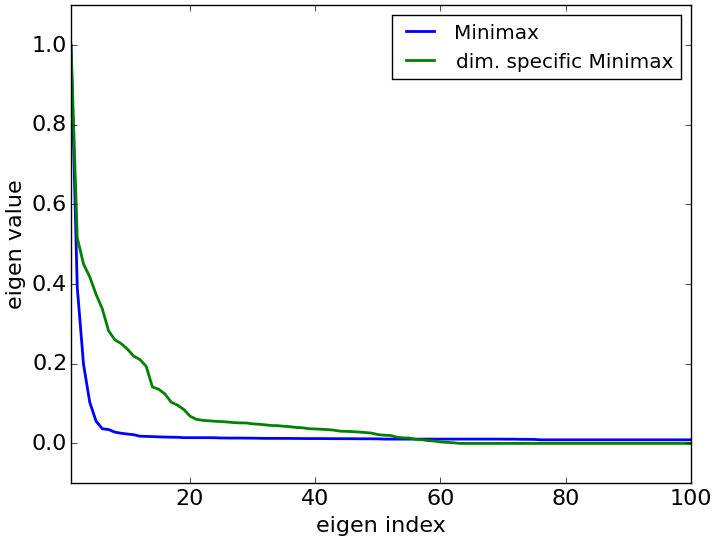}
        \label{fig:Contraceptive_Method_Choice_eigen}
    }
    \hspace{5mm}
    \subfigure[\emph{Banknote Authentication}]
    {
        \includegraphics[width=0.33\textwidth]{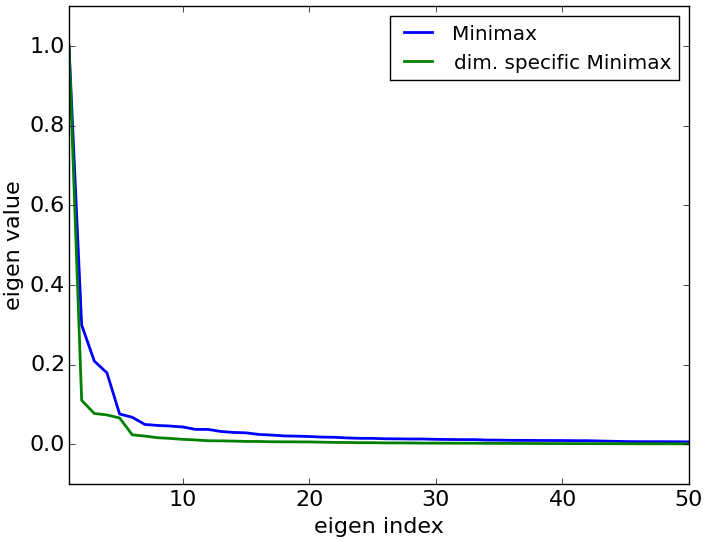}
        \label{fig:Banknote_Authentication_eigen}
    }    	
    \caption{Eigenspectra for different example datasets. A better Minimax variant often yields a sharper
eigenspectrum.}
    \label{fig:eigenspectra}
\end{figure*}

\begin{table*}[ht!]
\caption{Runtimes (in seconds) of FW-MM and MST-MM for computing pairwise Minimax distances on different real-world datasets. MST-MM performs significantly faster.} 
\centering 
\begin{tabular}{c | c c c c c c} 
\hline\hline 
dataset & \emph{Glass Iden.} & \emph{Haberman Surv.} & \emph{Hayes Roth} & \emph{Perfume} & \emph{Cloud} & \emph{Skin Seg.} \\ [0.5ex] 
\hline 
size       & 160 & 214 & 306 & 560 & 1024 & 2000 \\ 
FW-MM  & 0.030 & 0.067 & 0.230 & 1.399 & 18.54 & 127.12 \\
MST-MM & 0.019 & 0.024 & 0.080 & 0.150 & 0.239 & 0.540 \\
\hline 
\end{tabular}
\label{table:comparisonWarshall_Us} 
\end{table*}

\textbf{Efficiency.} As a side study, we also investigate the runtimes of computing pairwise Minimax distance via either based on computing an MST (MST-MM) or the Floyd-Warshall algorithm (FW-MM). The results (in seconds) have been shown in Table \ref{table:comparisonWarshall_Us} for some of the datasets. We observe that MST-MM yields a significantly faster approach to computer the pairwise Minimax distances. The difference is even more obvious when the dataset is larger.  We note that the both methods yield the same Minimax distances. We observe consistent results on the other datasets.

\subsection{Experiments on clustering of document scannings}

We  study the impact of performing Minimax distances on clustering of scannings of the documents collocated by a large document processing corporation. The dataset contains the vectors of $675$ documents each represented in a $4096$ dimensional space. The vectors contain metadata, the document layout, tags used in the document, the images in the document, the text in the document, the author information, the mathematical expressions and other structural information.
This dataset contains $56$ clusters  some of which have only one single document. We call this dataset \emph{dataset 1}.
Then, by removing the clusters with only one or two documents, we obtain \emph{dataset 2} which includes $634$ documents and $34$ clusters. Finally, we obtain \emph{dataset 3} which contains the clusters that have at least $5$ documents. This datasets consists of $592$ documents and $21$ clusters. We compute the pairwise distances of pairs of documents according to squared Euclidean distance.  The goal is to demonstrate the applicability of Gaussian Mixture Models (GMM) to Minimax distances, as well as to show if the use of Minimax distances improves the results.

\begin{table*}[ht!]
\caption{Performance of  GMM  on original vectors (\emph{base}), on the standard Minimax vectors (\emph{Minimax}), and on the dimension-specific Minimax vectors at subspaces (\emph{DimMM}) in the clustering task. The number at the front of \emph{DimMM} shows the dimensionality of the subspace used to compute the dimension-specific Minimax distances. Different methods are compared w.r.t. adjusted Rand score and adjusted Mutual Information criteria. We repeat the DimMM variant 10 times and report the mean(std) results. On all datasets, \emph{DimMM} yields  higher scores of the evaluation criteria.
}
\centering 
\begin{tabular}{|| c || c c || c c || c c ||} 
\hline\hline 
 &\multicolumn{2}{c||}{dataset 1}&\multicolumn{2}{c||}{dataset 2}&\multicolumn{2}{c||}{dataset 3}\\
 \hline
 method & Rand score & Mutual Info. & Rand score & Mutual Info. & Rand score & Mutual Info. \\ [0.5ex] 
\hline 
base & 0.2539 & 0.5914 & 0.3191 & 0.6478 & 0.4268 & 0.7195    \\  
\hline 
Minimax & 0.3033 & 0.6263 & 0.3995 & 0.6943 & 0.4611 & 0.7407\\  
\hline 
DimMM050 & 0.3605(0.016) & 0.6489(0.015) & 0.4306(0.016) & 0.7052(0.018) & 0.4556(0.014) & 0.7294(0.018)   \\  
DimMM075 & 0.3767(0.013) & 0.6431(0.019) & 0.4238(0.014) & 0.7003(0.019) & 0.5254(0.015) & 0.7587(0.018)    \\  
DimMM100 & 0.3728(0.014) & 0.6451(0.018) & 0.4126(0.011) & 0.6996(0.012) & 0.5002(0.013) & 0.7577(0.016)    \\  
DimMM125 & 0.3538(0.015) & 0.6470(0.015) & 0.4167(0.017) & 0.6985(0.019) & 0.4731(0.015) & 0.7448(0.019)   \\  
DimMM150 & 0.3431(0.012) & 0.6418(0.016) & 0.3932(0.013) &0.6960(0.014)  & 0.4917(0.012) & 0.7530(0.014)    \\  
DimMM175 & 0.3484(0.017) & 0.6459(0.016) & 0.3766(0.018) & 0.6896(0.019) & 0.4512(0.018) & 0.7375(0.017)    \\  
DimMM200 & 0.3351(0.015) & 0.6374(0.018) & 0.4023(0.011) & 0.6948(0.011) & 0.4607(0.014) & 0.7429(0.016)   \\  
DimMM250 & 0.3322(0.012) & 0.6413(0.016) & 0.3809(0.013) & 0.6866(0.0114) & 0.5034(0.011) & 0.7530(0.012)  \\  
DimMM300 & 0.3243(0.012) & 0.6376(0.015) & 0.3942(0.014) & 0.6954(0.017) & 0.4800(0.016) & 0.7469(0.015)  \\  
DimMM350 & 0.3332(0.016) & 0.6419(0.018) & 0.3831(0.012) & 0.6798(0.013) & 0.4660(0.011) & 0.7452(0.011)  \\  
DimMM400 & 0.3385(0.013) & 0.6406(0.013) & 0.4089(0.014 & 0.6944(0.015) & 0.4680(0.013) & 0.7429(0.012)  \\  
DimMM500 & 0.3272(0.015) & 0.6378(0.0179 & 0.4008(0.016) & 0.6929(0.019) & 0.4653(0.016) & 0.7470(0.017)  \\  
\hline 
\end{tabular}
\label{table:results_Xerox} 
\end{table*}

We apply dimension-specific Minimax distances to different subspaces of the original data. For this purpose, we define the parameter $b$ to be the number of dimensions of the subspace. Then, we randomly partition the original features (dimensions) such that the dimensionality of each subset equals $b$ (the last subset might have less dimensions than $b$). Note that different subsets have the same number of documents which equals to the number of documents in the original dataset. For each subset, we compute the pairwise Minimax distances and obtain multiple Minimax matrices. Thus, we apply the collective Minimax embedding to compute an embedding in a new space. Finally, we apply GMM to the Minimax vectors and compare with the results of GMM on original (base) vectors and GMM on standard Minimax vectors.

We use adjusted Rand score~\cite{hubert1985comparing} and adjusted Mutual Information~\cite{Vinh:2010} to evaluate the performance of different methods. Rand score computes the similarity between the estimated and the true clusterings. Mutual Information measures the mutual information between the two solutions.
Note that we compute the adjusted version of these criteria, such that they yield zero for random clustering.
Table~\ref{table:results_Xerox} shows the results on the three datasets. The number at the front of different dimension-specific (DimMM) variants indicates $b$. We repeat the DimMM variant 10 times and report the mean(std) results. We observe: i) computing Minimax distances enables GMM to better capture the underlying structures, thus yields improving the results. ii) The dimension-specific variant improves the clusters even further via extracting appropriate structures in different subspaces. iii) However, the first improvement, i.e., the use or not use of Minimax distances, has a more significant impact than the choice between the standard or the dimension-specific Minimax variants. Finally, we note that DimMM is equivalent to the standard Minimax if $b=4096$.

\subsection{Minimax $K$-NN classification}
We then examine the performance of our algorithm for $K$ nearest neighbor search with Minimax distances.
We perform our experiments on seven datasets,  four selected from 20 newsgroup collection and the others come from image and plant specification.
\\
(1) COMP: a subset of 20 newsgroup contains $2,936$ documents around \emph{computers}: `comp.graphics', `comp.os.ms-windows.misc', `comp.sys.ibm.pc.hardware', `comp.sys.mac.hardware', `comp.windows.x'. \\
(2) REC: a subset of 20 newsgroup with $2,389$ documents on \emph{sports}: `rec.autos', `rec.motorcycles', `rec.sport.baseball',  `rec.sport.hockey'. \\
(3)  SCI: a subset of 20 newsgroup having $2,373$ documents about \emph{science}: `sci.crypt' , `sci.electronics', `sci.med',  `sci.space'.  \\
(4)  TALK: a subset of 20 newsgroup with $1,952$ documents related to \emph{talk}: `talk.politics.guns', `talk.politics.mideast',  `talk.politics.misc', `talk.religion.misc'. \\
(5) IRIS: a common  dataset with $150$ samples from  three species `setosa', `virginica' and `versicolor'. \\
(6) OLVT: the Olivetti faces dataset from AT\&T which contains pictures from $40$ individuals from each $10$ pictures. The dimensionality is $4,096$. \\
(7) DGTS: images of $10$ digits each with $64$ dimensions ($1,797$ digits in total).

These datasets are well-known and publicly accessible, e.g. via \emph{sklearn.datasets}.\footnote{These datasets can be read and extracted via Python interfaces as described in \url{http://scikit-learn.org/stable/modules/classes.html\#module-sklearn.datasets}}
For the 20 newsgroup datasets, we obtain the vectors after eliminating the stop words. \emph{Cosine} similarity is known to be a better choice for textual data. Thus, for each pair of document vectors $\mathbf x_i$ and $\mathbf x_j$, we compute their \emph{cosine} similarity and then their dissimilarity by $1-cosine(\mathbf x_i,\mathbf x_j)$ to construct the pairwise dissimilarity matrix $\mathbf D$. For non-textual datasets, we directly compute the pairwise squared Euclidean distances. Thereby, we construct graph $\mathcal G(\mathbf O,\mathbf D)$. Note that the method in~\cite{KimC13AAAI} is not applicable to cosine-based dissimilarities.
We compare our algorithm (called PTMM) against the following methods.

\begin{enumerate}[leftmargin = *]
\item STND: Standard $K$-NN with the basic dissimilarities $\mathbf D$.
\item MPMM: Minimax $K$-NN proposed in~\cite{KimC07icml}.
\item Link: Link-based $K$-NN based on shortest distance algorithm~\cite{Dijkstra:ShortestPath}.
\end{enumerate}

$L^+$~\cite{Fouss:2012,YenSMS08} is a link-based method computed based on the pseudo inverse of Laplacian matrix.
However, this method has two main deficiencies: i) it is very slow; as mentioned earlier, its runtime is $\mathcal O(N^3)$ (see Table~\ref{table:runtime_L+} for runtime results on 20 newsgroup datasets), and ii) in the way that it is used (e.g.~\cite{YenSMS08,KimC13AAAI}) $L^+$ is computed for all data, including training and test objects, i.e. the test objects are not completely separated from training. The correct approach would compute the pseudo inverse of the Laplacian only for the training dataset and then would extend it to contain the out-of-sample object. 

We perform the experiments and compute accuracy in \emph{leave-one-out} manner, i.e. the $i$-th object is left out and its label is predicted via the other objects. In the final voting, the nearest neighbors are weighted according to inverse of their  dissimilarities with the test object. We calculate the accuracy by averaging over all $N$ objects. $M$-fold cross validation could be used instead of the \emph{leave-one-out} approach. However, the respective ranking of the runtimes of different methods is independent of the way the accuracy is computed.

\begin{figure*}[ht!]
    \centering
    \subfigure[Runtime for COMP]
    {
        \includegraphics[width=0.34\textwidth]{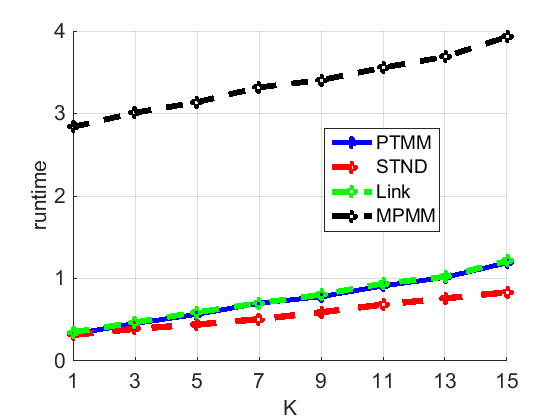}
        \label{fig:Time_COMP}
    }
    \hspace{5mm}
    \subfigure[Runtime for REC]
    {
        \includegraphics[width=0.34\textwidth]{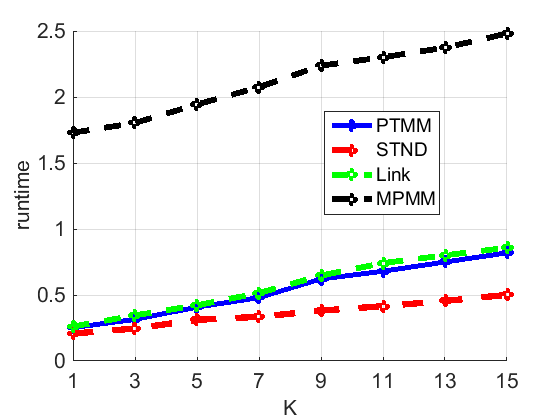}
        \label{fig:Time_REC}
    }
    \\
    \subfigure[Runtime for SCI]
    {
        \includegraphics[width=0.34\textwidth]{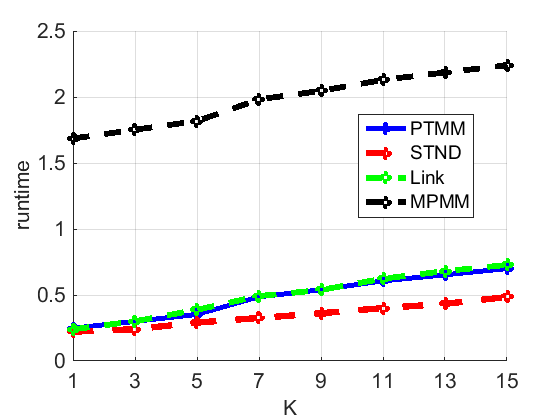}
        \label{fig:Time_SCI}
    }
    \hspace{5mm}
    \subfigure[Runtime for TALK]
    {
        \includegraphics[width=0.34\textwidth]{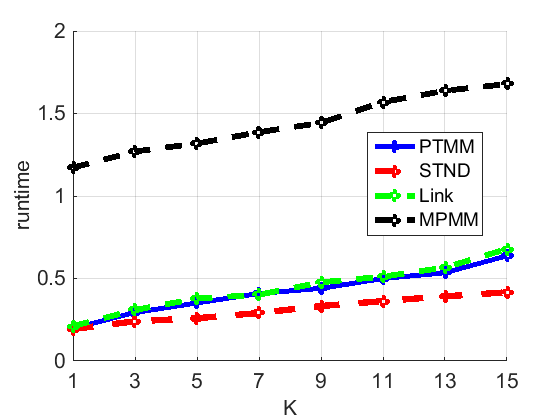}
        \label{fig:Time_TALK}
    }
    \caption{Comparison of  runtime (in seconds)  of different methods on 20 news group datasets. PTMM is significantly faster than MPMM and is only slightly slower than STND. Figures~\ref{fig:Acc_COMP} and \ref{fig:Acc_REC} show the accuracy scores, where Minimax $K$-NN gives the best results.}
    \label{fig:Runtime_Res_20News}
\end{figure*}

\begin{figure*}[ht!]
    \centering
     \subfigure[Accuracy for COMP]
    {
        \includegraphics[width=0.34\textwidth]{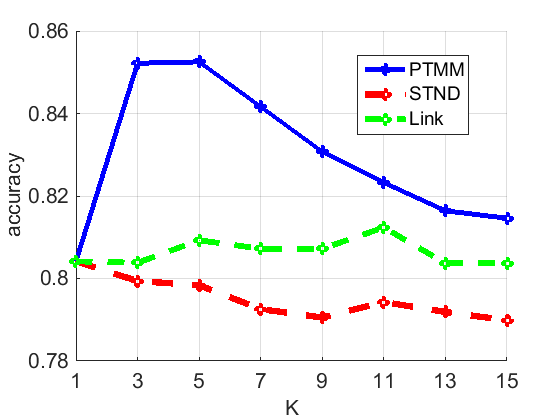}
        \label{fig:Acc_COMP}
    }
    \hspace{5mm}
    \subfigure[Accuracy for REC]
    {
        \includegraphics[width=0.34\textwidth]{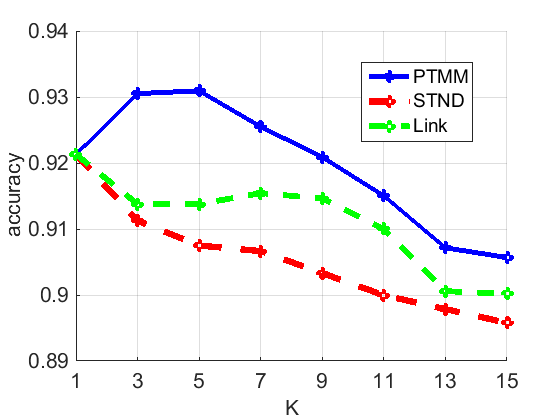}
        \label{fig:Acc_REC}
    }
    \\
    \subfigure[Accuracy for SCI]
    {
        \includegraphics[width=0.34\textwidth]{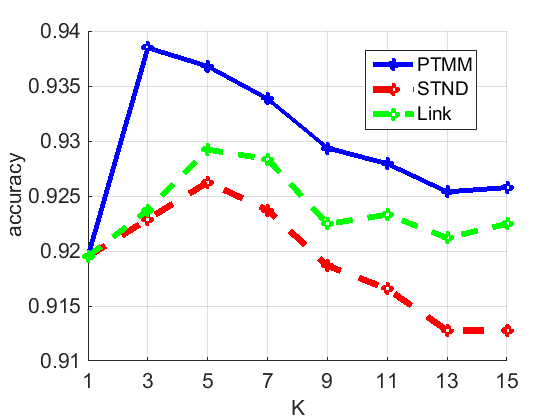}
        \label{fig:Acc_SCI}
    }
    \hspace{5mm}
    \subfigure[Accuracy for TALK]
    {
        \includegraphics[width=0.34\textwidth]{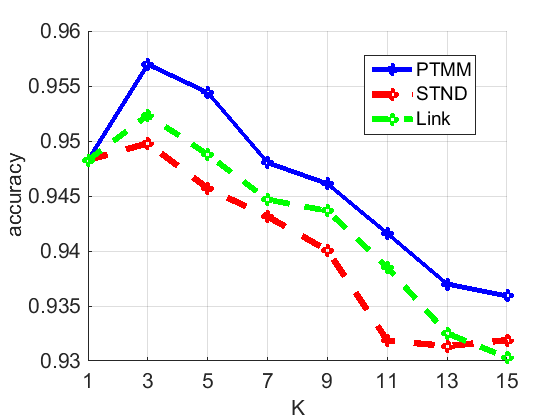}
        \label{fig:Acc_TALK}
    }
    \caption{Accuracy scores of different $K$-NN methods on 20 new group datasets, where Minimax $K$-NN gives the best results.}
    \label{fig:Acc_Res_20News}
\end{figure*}

\begin{figure*}[!ht]
    \centering
     \subfigure[Runtime for IRIS]
    {
    \includegraphics[width=0.31\textwidth]{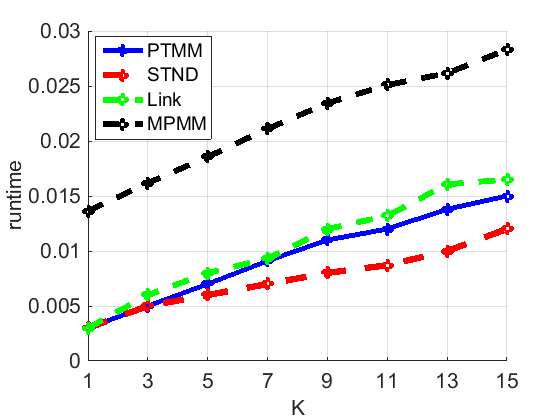}
        \label{fig:Time_IRIS}
    }
    \subfigure[Runtime for OLVT]
    {
        \includegraphics[width=0.31\textwidth]{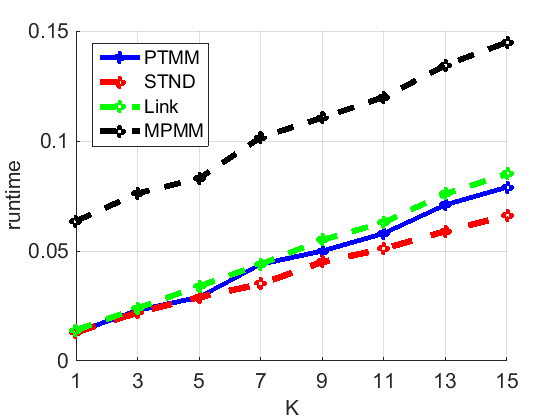}
        \label{fig:Time_OLVT}
    }
    \subfigure[Runtime for DGTS]
    {
        \includegraphics[width=0.31\textwidth]{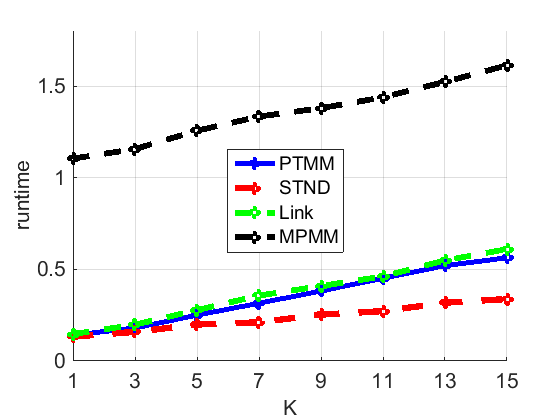}
        \label{fig:Time_DGTS}
    }
    \caption{Comparison of runtime (in seconds) of different methods on non-textual datasets. PTMM runs much faster than MPMM.}
    \label{fig:Runtime_Res_other}
\end{figure*}

\begin{table}[ht!]
\caption{Runtime (in seconds) of computing $L^+$ for different 20 news group datasets. $L^+$ is considerably slower.} 
\centering 
\begin{tabular}{c | c c c c} 
\hline\hline 
dataset & COMP & REC & SCI & TALK \\ 
\hline 
runtime & 20.67 & 12.93 & 11.84 & 6.55 \\  
\hline 
\end{tabular}
\label{table:runtime_L+} 
\end{table}

Figure~\ref{fig:Runtime_Res_20News} illustrates  the runtime of different measures on 20 news group datasets.
Our algorithm for computing the Minimax distances (i.e. PTMM) is significantly faster than MPMM. PTMM is only slightly slower than the standard method and is even slightly faster than Link. The main source of computational inefficiency of MPMM comes from requiring a MST which needs an $\mathcal O(N^2)$ runtime.  PTMM theoretically has the same complexity as STND (i.e. $\mathcal O(N)$) and in practice performs only slightly slower which is due to one additional update step.  As mentioned $L^+$ is considerably slower, even compared to MPMM (see Table~\ref{table:runtime_L+}).

To demonstrate the generality of our algorithm, we have performed experiments on datasets from other domains. Figure~\ref{fig:Runtime_Res_other} shows the runtime of different methods on non-textual datasets (IRIS, OLVT, and DGTS).
We observe a consistent performance behavior:  PTMM runs only slightly slower than STND and much faster than MPMM. Link is slightly slower than PTMM.

The usefulness of Minimax distances for $K$-NN classification has been investigated in previous studies, e.g. in~\cite{KimC07icml,KimC13AAAI}. However, for the proof of concept, in Figures~\ref{fig:Acc_Res_20News} and~\ref{fig:Accuracy_Res_other}, we illustrate the accuracy of different methods on the textual and non-textual datasets. Note that PTMM and MPMM yield the same accuracy scores. They differ only in the way they compute the Minimax nearest neighbors of $v$, thereby we  have compared them only w.r.t runtime. Consistent to the previous study, we observe that Minimax measure is always the best choice, i.e. it yields the highest accuracy. We emphasize that this is the first study that demonstrates the effectiveness and performance of Minimax $K$-nearest neighbor classification beyond Euclidean distances (i.e. \emph{cosine} similarity, a measure which is known to be more appropriate for textual data). As mentioned earlier, the previous studies \cite{KimC07icml,KimC13AAAI} always build the Minimax distances on squared Euclidean measures.

\begin{figure}[!t]
    \centering
    \subfigure[Accuracy for IRIS]
    {
        \includegraphics[width=0.31\textwidth]{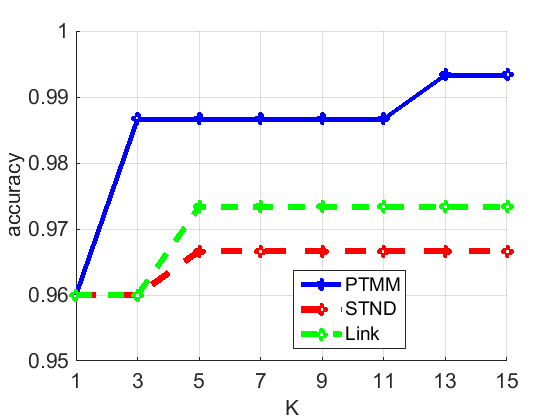}
        \label{fig:Acc_IRIS}
    }
    \subfigure[Accuracy for OLVT]
    {
        \includegraphics[width=0.31\textwidth]{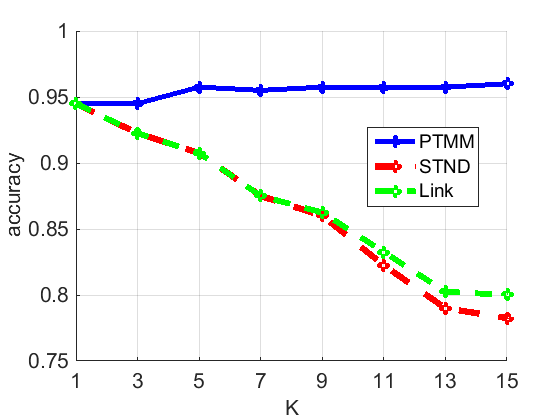}
        \label{fig:Acc_OLVT}
    }
    \subfigure[Accuracy for DGTS]
    {
        \includegraphics[width=0.31\textwidth]{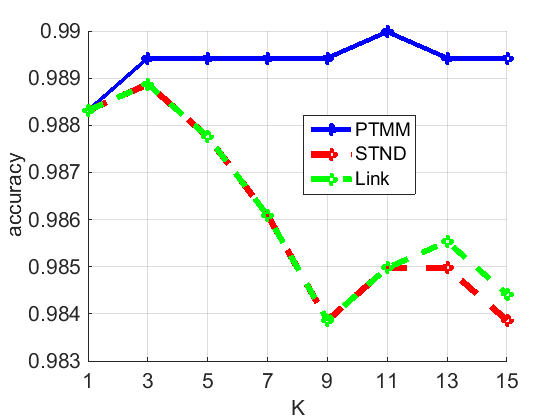}
        \label{fig:Acc_DGTS}
    }
    \caption{Comparison of accuracy of different methods on IRIS, OLVT and DGTS datasets. Using Minimax distance measure improves the results.}
    \label{fig:Accuracy_Res_other}
\end{figure}

\textbf{Scalability.} Finally, we study the scalability of different methods on larger datasets from two-moons data. As shown in~\cite{KimC13AAAI} only $L^+$ and Minimax measures  perform correctly on this data which has a non-linear geometric structure. Table~\ref{table:runtime_2moon} compares the runtime of these two methods for $K=5$.  PTMM performs significantly faster than the alternatives, e.g. it is  three times faster than MPMM and $250$ times faster than $L^+$ for $10,000$ objects.

\begin{table}[ht!]
\caption{Runtime (in seconds) of different methods on two-moon datasets. PTMM performs significantly faster. PTMM and MPMM produce the same results.} 
\centering 
\begin{tabular}{c | c c c c} 
\hline\hline 
size & 5,000 & 10,000 & 15,000 & 20,000 \\ [0.5ex] 
\hline 
PTMM & 1.03 & 4.42 & 10.96 & 19.187 \\  
MPMM & 2.43 & 13.15 & 34.72 & 58.48 \\  
$L^+$ & 107.61 &	1105.13 & 2855.1 & 6603.16 \\
\hline 
\end{tabular}
\label{table:runtime_2moon} 
\end{table}

\subsection{Outlier detection along with Minimax $K$-NN search}
At the end, we experimentally investigate the ability of Algorithm~\ref{alg:PathOutlier} for detecting outliers while performing $K$-nearest neighbor search. We run our experiments on the textual datasets, i.e., on COMP, REC, SCI, and TALK.
We add the unused documents of  20 news group collection to each dataset to play the role of outliers or the mismatching data source. This \emph{ourlier dataset} consists of the $1,664$ documents of  `alt.atheism', `misc.forsale', `soc.religion.christian' categories. For each document, either from train or from outlier set, we compute its $K$ nearest neighbors  only from the documents of the train dataset and check whether it is predicted as an outlier or not.

We compare our algorithm (called here PTMMout), against the two others: STND and Link.
These two algorithms do not directly provide outlier prediction. Thereby, in order to perform a fair comparison, we compute a Prim minimum spanning tree over the set containing $v$ and the objects selected by the $K$-NN search, where $v$ is the initial node. Then we check the same conditions to determine whether $v$ is determined as an outlier. We call these two variants respectively STNDout and LINKout.

We obtain the rate of correct outlier prediction for the outlier dataset, called \emph{rate\_1}, by the ratio of the number of correct outlier reports to the size of this set. However, an algorithm might report an artificially high outlier number. Thereby, similar to \emph{precision-recall} trade-off, we investigate the train dataset too and compute the ratio of the number of \emph{non-outlier} reports to the total size of train dataset and call it \emph{rate\_2}. Finally, similar to F-measure, we compute the harmonic average of \emph{rate\_1} and \emph{rate\_2} and call it \emph{o\_rate}, i.e.
\begin{equation}
	o\_rate = 2 \cdot \frac{rate\_1 \cdot rate\_2}{rate\_1+rate\_2} \, .
\end{equation}

\begin{figure*}[ht!]
    \centering
    \subfigure[\emph{o\_rate} for COMP]
    {
        \includegraphics[width=0.35\textwidth]{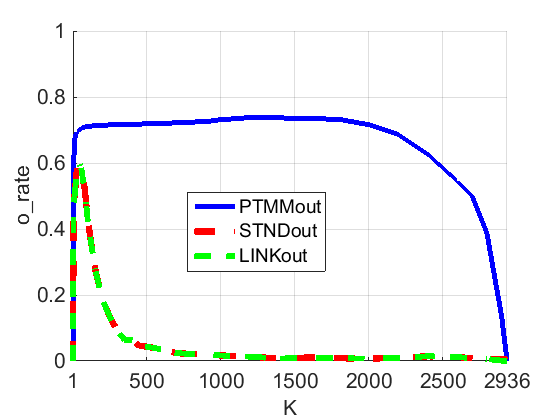}
        \label{fig:Outlier_COMP}
    }
    \hspace{5mm}
    \subfigure[\emph{o\_rate} for REC]
    {
        \includegraphics[width=0.35\textwidth]{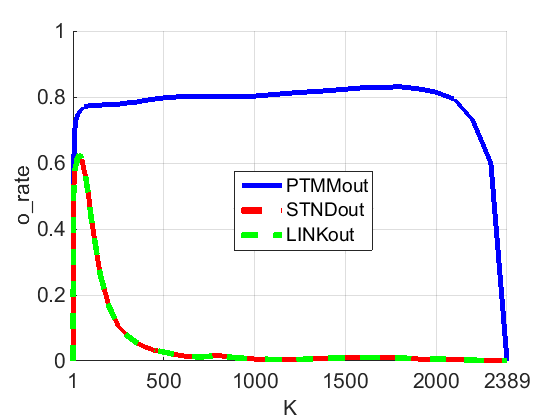}
        \label{fig:Outlier_REC}
    }
    \\
    \subfigure[\emph{o\_rate} for SCI]
    {
        \includegraphics[width=0.35\textwidth]{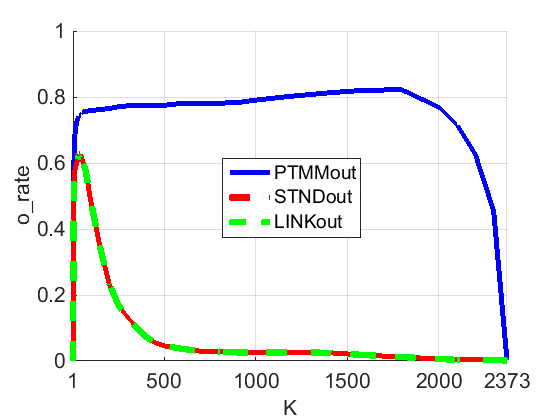}
        \label{fig:Outlier_SCI}
    }
    \hspace{5mm}
    \subfigure[\emph{o\_rate} for TALK]
    {
        \includegraphics[width=0.35\textwidth]{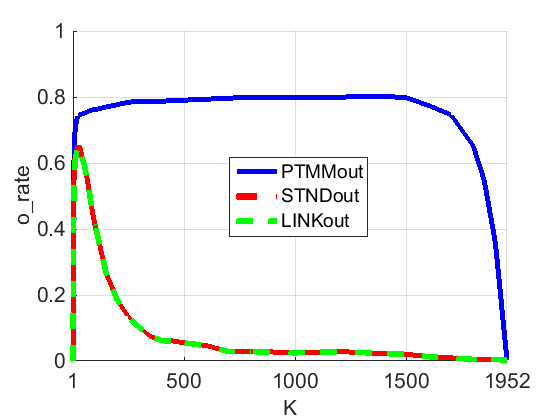}
        \label{fig:Outlier_TALK}
    }
    \caption{\emph{o\_rate} score for different text datasets when PTMMout, STNDout and LINKout are used to predict outliers. PTMMout yields significantly a higher and more stable \emph{o\_rate}. Figures~\ref{fig:Outlier_COMPTest} and~\ref{fig:Outlier_COMPTrain} show \emph{rate\_1} and \emph{rate\_2} for COMP.}
    \label{fig:Outlier_20News}
\end{figure*}

\begin{figure*}[ht!]
    \centering
    \subfigure[\emph{rate\_1} for COMP]
    {
        \includegraphics[width=0.35\textwidth]{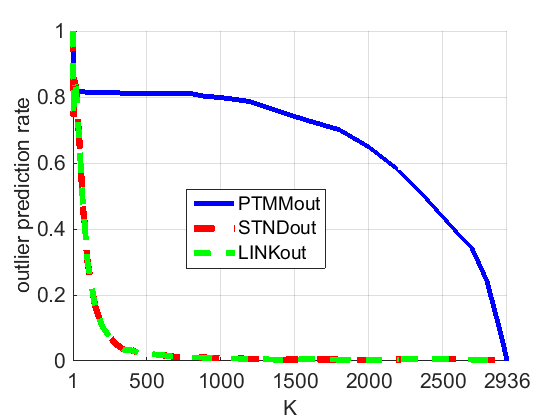}
        \label{fig:Outlier_COMPTest}
    }
    \hspace{5mm}
    \subfigure[\emph{rate\_2} for COMP]
    {
        \includegraphics[width=0.35\textwidth]{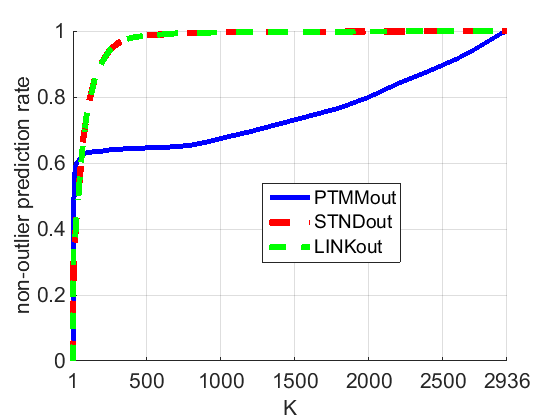}
        \label{fig:Outlier_COMPTrain}
    }
    \caption{\emph{rate\_1} and \emph{rate\_2} for COMP. Only PTMMout prevents sharp changes. }
    \label{fig:Outlier_20News_detail}
\end{figure*}

Figure~\ref{fig:Outlier_20News} shows the \emph{o\_rate} scores for different datasets and for different values of $K$. For all datasets, we observe that PTMMout yields significantly a higher \emph{o\_rate} compared to the other algorithms. Moreover, the \emph{o\_rate} of PTMMout is considerably stable and constant over different $K$, i.e. the choice of optimal $K$ is much less critical. For STNDout and LINKout, \emph{o\_rate} is maximal for only a fairly small range of $K$ (e.g. between $30$ and $50$) and then it sharply drops down to zero. Figure \ref{fig:Outlier_20News_detail} illustrates an in-depth analysis of \emph{rate\_1} and \emph{rate\_2} for the COMP dataset. Generally, as we increase $K$, \emph{rate\_1} decreases but \emph{rate\_2} increases. The reason is that as we involve more neighbors in $\mathcal {NP}(v)$, then the chance of having an edge whose both sides are inside $\mathcal {NP}(v)$ and its weight represents a Minimax distance increases. This case might particularly occur when $\mathcal {NP}(v)$ contains objects from more than one class. Thereby, the probability that an object is labeled as a \emph{non-outlier} increases. However, this transition is very sharp for STNDout and LINKout, such that for a relatively small $K$ they label every object (either from train set or from outlier set) as a \emph{non-outlier}. Moreover, as mentioned before, even for small $K$, STNDout and LINKout yield significantly lower \emph{o\_rate} than PTMMout.

\section{Conclusion}\label{sec:conclusion}

We developed a framework to apply Minimax distances to any learning algorithm that works on numerical data, which takes into account both generality and efficiency.
We studied both computing the pairwise Minimax distances for all pairs of objects and as well as computing the Minimax distances of all the objects to/from a fixed (test) object.
For the \emph{all-pair} case, we first employed the equivalence of Minimax distances over a graph and over a minimum spanning tree constructed on that, and developed an efficient algorithm to compute the pairwise Minimax distances on a tree. Then, we studied computing an embedding of the objects into a new vector space, such that their pairwise Minimax distance in the original data space equals to their squared Euclidean distance in the new space. We then extended our approach to the cases wherein there are multiple pairwise Minimax matrices.
In the following,  we studied computing Minimax distances from a fixed (test) object which can be used for instance in $K$-nearest neighbor search.
Moreover, we investigated in detail the edges selected by the Minimax distances and thereby augmented the Minimax $K$-nearest neighbor search with the ability of detecting outlier objects.
 For each setting, we demonstrated the effectiveness of our framework via several experimental studies.

\section*{Acknowledgment}
This work was partially supported by the Wallenberg AI, Autonomous Systems and Software Program (WASP) funded
by the Knut and Alice Wallenberg Foundation.
Parts of this work were accomplished at Xerox Research Centre Europe (XRCE).

\bibliographystyle{plain}
\bibliography{references}

\begin{thebibliography}{10}

\bibitem{Aho:1974}
Alfred~V. Aho and John~E. Hopcroft.
\newblock {\em The Design and Analysis of Computer Algorithms}.
\newblock Addison-Wesley Longman Publishing Co., Inc., Boston, MA, USA, 1st
  edition, 1974.

\bibitem{Chang:2008:RPS}
Hong Chang and Dit-Yan Yeung.
\newblock Robust path-based spectral clustering.
\newblock {\em Pattern Recogn.}, 41(1):191--203, 2008.

\bibitem{Chebotarev:2011}
Pavel Chebotarev.
\newblock A class of graph-geodetic distances generalizing the shortest-path
  and the resistance distances.
\newblock {\em Discrete Appl. Math.}, 159(5):295--302, 2011.

\bibitem{Cormen:2001:IA:580470}
Thomas~H. Cormen, Clifford Stein, Ronald~L. Rivest, and Charles~E. Leiserson.
\newblock {\em Introduction to Algorithms}.
\newblock McGraw-Hill Higher Education, 2nd edition, 2001.

\bibitem{Dijkstra:ShortestPath}
E.~W. Dijkstra.
\newblock A note on two problems in connexion with graphs.
\newblock {\em Numerische Mathematik}, 1:269--271, 1959.

\bibitem{Fiedler1998}
Miroslav Fiedler.
\newblock Ultrametric sets in euclidean point spaces.
\newblock {\em ELA. The Electronic Journal of Linear Algebra}, 3:23--30, 1998.

\bibitem{FischerB03}
Bernd Fischer and Joachim~M. Buhmann.
\newblock Path-based clustering for grouping of smooth curves and texture
  segmentation.
\newblock {\em IEEE Trans. Pattern Anal. Mach. Intell.}, 25(4):513--518, 2003.

\bibitem{Fouss:2012}
Fran\c{c}ois Fouss, Kevin Francoisse, Luh Yen, Alain Pirotte, and Marco
  Saerens.
\newblock An experimental investigation of kernels on graphs for collaborative
  recommendation and semisupervised classification.
\newblock {\em Neural Networks}, 31:5372, 2012.

\bibitem{Fouss:2007}
Francois Fouss, Alain Pirotte, Jean-Michel Renders, and Marco Saerens.
\newblock Random-walk computation of similarities between nodes of a graph with
  application to collaborative recommendation.
\newblock {\em IEEE Trans. on Knowl. and Data Eng.}, 19(3):355--369, 2007.

\bibitem{Gabow1986}
H~N Gabow, Z~Galil, T~Spencer, and R~E Tarjan.
\newblock Efficient algorithms for finding minimum spanning trees in undirected
  and directed graphs.
\newblock {\em Combinatorica}, 6(2):109--122, 1986.

\bibitem{GlobersonCPT07}
Amir Globerson, Gal Chechik, Fernando Pereira, and Naftali Tishby.

\bibitem{Chehreghani16MLj}
Morteza Haghir~Chehreghani.
\newblock Adaptive trajectory analysis of replicator dynamics for data
  clustering.
\newblock {\em Machine Learning}, 104(2-3):271--289, 2016.

\bibitem{ChehreghaniSDM2016}
Morteza Haghir~Chehreghani.
\newblock K-nearest neighbor search and outlier detection via minimax
  distances.
\newblock In {\em Siam International Conference on Data Mining}. Siam, 2016.

\bibitem{ChehreghaniAAAI17}
Morteza Haghir~Chehreghani.
\newblock Classification with minimax distance measures.
\newblock In {\em Proceedings of the Thirty-First {AAAI} Conference on
  Artificial Intelligence}, pages 1784--1790, 2017.

\bibitem{Chehreghani17ICDM}
Morteza Haghir~Chehreghani.
\newblock Efficient computation of pairwise minimax distance measures.
\newblock In {\em 2017 {IEEE} International Conference on Data Mining, {ICDM}},
  pages 799--804. {IEEE} Computer Society, 2017.

\bibitem{abs-1904-13223}
Morteza Haghir~Chehreghani.
\newblock Hierarchical correlation clustering and tree preserving embedding.
\newblock {\em CoRR}, abs/2002.07756, 2020.

\bibitem{Hofmann06areview}
Thomas Hofmann, Bernhard Sch\"{o}lkopf, and Alexander~J. Smola.
\newblock Kernel methods in machine learning.
\newblock {\em Annals of Statistics}, 36(3):1171--1220, 2008.

\bibitem{Horn:MA:5509}
Roger~A. Horn and Charles~R. Johnson, editors.
\newblock {\em Matrix Analysis}.
\newblock Cambridge University Press, 1990.

\bibitem{Hu61}
T.C. Hu.
\newblock The maximum capacity route problem.
\newblock {\em Operations Research}, 9:898--900, 1961.

\bibitem{hubert1985comparing}
L.~Hubert and P.~Arabie.
\newblock {Comparing partitions}.
\newblock {\em Journal of classification}, 2(1):193--218, 1985.

\bibitem{KhoshneshinS10}
Mohammad Khoshneshin and W.~Nick Street.
\newblock Collaborative filtering via euclidean embedding.
\newblock In {\em Proceedings of the 2010 {ACM} Conference on Recommender
  Systems, RecSys 2010, Barcelona, Spain, September 26-30, 2010}, pages 87--94,
  2010.

\bibitem{KimC07icml}
Kye-Hyeon Kim and Seungjin Choi.
\newblock Neighbor search with global geometry: a minimax message passing
  algorithm.
\newblock In {\em ICML}, pages 401--408, 2007.

\bibitem{KimC13AAAI}
Kye-Hyeon Kim and Seungjin Choi.
\newblock Walking on minimax paths for k-nn search.
\newblock In {\em AAAI}, 2013.

\bibitem{Kschischang:2006}
F.~R. Kschischang, B.~J. Frey, and H.~A. Loeliger.
\newblock Factor graphs and the sum-product algorithm.
\newblock {\em IEEE Trans. Inf. Theor.}, 47(2):498--519, 2006.

\bibitem{Leclerc1981}
Bruno Leclerc.
\newblock Description combinatoire des ultramétriques.
\newblock {\em Mathématiques et Sciences Humaines}, 73:5--37, 1981.

\bibitem{Lichman:2013}
M.~Lichman.
\newblock {UCI} machine learning repository, 2013.

\bibitem{Luxburg:2007}
Ulrike Luxburg.
\newblock A tutorial on spectral clustering.
\newblock {\em Statistics and Computing}, 17(4):395--416, 2007.

\bibitem{doi:10.2307/2348634}
A.~Mead.
\newblock Review of the development of multidimensional scaling methods.
\newblock {\em Journal of the Royal Statistical Society: Series D (The
  Statistician)}, 41(1):27--39.

\bibitem{Nadler07}
Boaz Nadler and Meirav Galun.
\newblock Fundamental limitations of spectral clustering.
\newblock In {\em in Advanced in Neural Information Processing Systems 19},
  pages 1017--1024, 2007.

\bibitem{PavanP07}
Massimiliano Pavan and Marcello Pelillo.
\newblock Dominant sets and pairwise clustering.
\newblock {\em {IEEE} Trans. Pattern Anal. Mach. Intell.}, 29(1):167--172,
  2007.

\bibitem{Prim1957}
Robert~C. Prim.
\newblock Shortest connection networks and some generalizations.
\newblock {\em The Bell Systems Technical Journal}, 36(6):1389--1401, 1957.

\bibitem{Quarteroni2007}
Alfio Quarteroni, Riccardo Sacco, and Fausto Saleri.
\newblock {\em Approximation of Eigenvalues and Eigenvectors}.
\newblock 2007.

\bibitem{Schoenberg}
I.~J. Schoenberg.
\newblock On certain metric spaces arising from euclidean spaces by a change of
  metric and their imbedding in hilbert space.
\newblock {\em Annals of Mathematics}, 38(4):787--793, 1937.

\bibitem{KernelbookShaweTaylor}
John Shawe-Taylor and Nello Cristianini.
\newblock {\em Kernel Methods for Pattern Analysis}.
\newblock Cambridge University Press, 2004.

\bibitem{tenenbaum_global_2000}
Joshua~B. Tenenbaum, Vin de~Silva, and John~C. Langford.
\newblock A global geometric framework for nonlinear dimensionality reduction.
\newblock {\em Science}, 290(5500):2319, 2000.

\bibitem{torgerson1958theory}
W.S. Torgerson.
\newblock {\em Theory and methods of scaling}.
\newblock Wiley, 1958.

\bibitem{Varga1993OnSU}
Richard~S. Varga and Reinhard Nabben.
\newblock On symmetric ultrametric matrices.
\newblock In {\em Numerical Linear Algebra}, 1993.

\bibitem{Veenman02amaximum}
C.~J. Veenman, M.J.T. Reinders, and E.~Backer.
\newblock A maximum variance cluster algorithm.
\newblock {\em IEEE Trans. Pattern Anal. Mach. Intell}, 24, 2002.

\bibitem{Vinh:2010}
Nguyen~Xuan Vinh, Julien Epps, and James Bailey.
\newblock Information theoretic measures for clusterings comparison: Variants,
  properties, normalization and correction for chance.
\newblock {\em J. Mach. Learn. Res.}, 11:2837--2854, 2010.

\bibitem{Weinberger:2009}
Kilian~Q. Weinberger and Lawrence~K. Saul.
\newblock Distance metric learning for large margin nearest neighbor
  classification.
\newblock {\em J. Mach. Learn. Res.}, 10:207--244, 2009.

\bibitem{YenSMS08}
Luh Yen, Marco Saerens, Amin Mantrach, and Masashi Shimbo.
\newblock A family of dissimilarity measures between nodes generalizing both
  the shortest-path and the commute-time distances.
\newblock In {\em KDD}, pages 785--793, 2008.

\bibitem{RePEc1938}
Gale Young and A.~Householder.
\newblock Discussion of a set of points in terms of their mutual distances.
\newblock {\em Psychometrika}, 3(1):19--22, 1938.

\bibitem{ZadehB09}
Reza Zadeh and Shai Ben-David.
\newblock A uniqueness theorem for clustering.
\newblock In {\em UAI}, pages 639--646, 2009.

\end{thebibliography}

\end{document}